\PassOptionsToPackage{numbers, compress}{natbib}
\documentclass{article}
\usepackage{microtype}
\usepackage{graphicx}
\usepackage{subfigure}
\usepackage{booktabs} 
\usepackage{amsmath}
\usepackage[table]{xcolor}
\usepackage{amsfonts}
\usepackage{amsthm}
\usepackage{paralist}
\usepackage{hyperref}
\usepackage[subtle]{savetrees}

\usepackage[accepted]{icmlw2019generalization}

\def\theTitle{Adversarial Training Can Hurt Generalization}
\icmltitlerunning{\theTitle}

\newcommand{\wrn}[2]{\text{WRN#1-#2}}

\newcommand\sF{\ensuremath{\mathcal{F}}}

\newcommand\sN{\ensuremath{\mathcal{N}}}

\newcommand\sX{\ensuremath{\mathcal{X}}}
\newcommand\sY{\ensuremath{\mathcal{Y}}}

\newcommand\by{\ensuremath{\mathbf{y}}}

\newcommand\half{\ensuremath{\frac{1}{2}}}
\newcommand\R{\ensuremath{\mathbb{R}}} 
\newcommand\refeqn[1]{(\ref{eqn:#1})}
\newcommand\refeqns[2]{(\ref{eqn:#1}) and (\ref{eqn:#2})}

\newcommand\refsec[1]{Section~\ref{sec:#1}}

\newcommand\reftab[1]{Table~\ref{tab:#1}}

\newcommand\refthm[1]{Theorem~\ref{thm:#1}}

\ifthenelse{\isundefined{\definition}}{\newtheorem{definition}{Definition}}{}
\ifthenelse{\isundefined{\assumption}}{\newtheorem{assumption}{Assumption}}{}
\ifthenelse{\isundefined{\hypothesis}}{\newtheorem{hypothesis}{Hypothesis}}{}
\ifthenelse{\isundefined{\proposition}}{\newtheorem{proposition}{Proposition}}{}
\ifthenelse{\isundefined{\theorem}}{\newtheorem{theorem}{Theorem}}{}
\ifthenelse{\isundefined{\lemma}}{\newtheorem{lemma}{Lemma}}{}
\ifthenelse{\isundefined{\corollary}}{\newtheorem{corollary}{Corollary}}{}
\ifthenelse{\isundefined{\alg}}{\newtheorem{alg}{Algorithm}}{}
\ifthenelse{\isundefined{\example}}{\newtheorem{example}{Example}}{}
\newcommand{\E}{\ensuremath{\mathbb{E}}} 

\newcommand{\fstar}{\ensuremath{f^\star}}

\DeclareMathOperator*{\argmax}{arg\,max}
\DeclareMathOperator*{\argmin}{arg\,min}
\newcommand{\RN}{\mathbb{R}}

\newcommand{\X}{\mathcal{X}}

\newcommand{\std}{\text{std}}

\newcommand{\truef}{f^\star}

\newcommand{\Prob}{\mathbb{P}}

\newcommand{\stest}{\hat{f}^\text{std}_n}
\newcommand{\robest}{\hat{f}^\text{rob}_n}

\newcommand{\cifar}{\textsc{Cifar-10}}

\newcommand{\mnist}{\textsc{Mnist}~}
\newcommand{\xline}{\sX_\text{line}}
\newcommand{\round}[1]{\lfloor{#1}\rceil}

\begin{document}

\twocolumn[
\icmltitle{\theTitle}



\icmlsetsymbol{equal}{*}

\begin{icmlauthorlist}
\icmlauthor{Aditi Raghunathan$^*$}{stanford}
\icmlauthor{Sang Michael Xie$^*$}{stanford}
\icmlauthor{Fanny Yang}{stanford}
\icmlauthor{John C. Duchi}{stanford}
\icmlauthor{Percy Liang}{stanford}
\end{icmlauthorlist}
\icmlaffiliation{stanford}{Stanford University}
\icmlcorrespondingauthor{Aditi Raghunathan}{aditir@cs.stanford.edu}
\icmlcorrespondingauthor{Sang Michael Xie}{xie@cs.stanford.edu}
\vskip 0.3in
]



\printAffiliationsAndNotice{\icmlEqualContribution} 

\begin{abstract}
While adversarial training can improve robust accuracy (against an adversary),
it sometimes hurts standard accuracy (when there is no adversary).
Previous work has studied this tradeoff between standard and robust accuracy,
but only in the setting where no predictor performs well on both objectives in the infinite data limit.
In this paper, we show that even when the optimal predictor with infinite data performs well on both objectives, a tradeoff can still manifest itself with finite data.
Furthermore, since our construction is based on a convex learning problem,
we rule out optimization concerns, thus laying bare a fundamental tension between robustness and generalization. Finally, we show that robust self-training mostly eliminates this tradeoff by leveraging unlabeled data.
\end{abstract}

\section{Introduction}
\label{sec:intro}
Neural networks trained using standard training
have very low accuracies on perturbed inputs commonly referred to as
\emph{adversarial examples}~\citep{szegedy2014intriguing}.
Even though adversarial training
\citep{goodfellow2015explaining,madry2018towards} can be effective at
improving the accuracy on such examples (\emph{robust accuracy}),
these modified training methods decrease accuracy on
natural unperturbed inputs (\emph{standard accuracy})
\citep{madry2018towards, zhang2019theoretically}.
\reftab{main-observation} shows the discrepancy between standard and adversarial training on \cifar. While adversarial training improves robust accuracy from 3.5\% to 45.8\%, standard accuracy drops from 95.2\% to 87.3\%.

One explanation for a tradeoff is that the standard and robust objectives
are fundamentally at conflict. Along these lines,~\citet{tsipras2019robustness} and \citet{zhang2019theoretically}
construct learning problems where the perturbations can change the output of the Bayes estimator. Thus no predictor can achieve both optimal standard accuracy and robust
accuracy even in the \emph{infinite data limit}.
However, we typically consider perturbations (such as imperceptible $\ell_\infty$ perturbations) which do not change the output of the Bayes estimator, so that a predictor with both optimal standard and high robust accuracy exists.

Another explanation could be that the hypothesis class is not rich enough to contain predictors that have optimal standard and high robust accuracy, even if they exist~\citep{nakkiran2019adversarial}. However, \reftab{main-observation} shows that adversarial training achieves 100\% standard and robust accuracy on the training set, suggesting that the hypothesis class is expressive enough in practice.

Having ruled out a conflict in the objectives and expressivity issues,~\reftab{main-observation} suggests that the tradeoff stems from the worse generalization of adversarial training either due to (i) the statistical properties of the robust objective or (ii) the dynamics of optimizing the robust objective on neural networks.
In an attempt to disentangle optimization and statistics, we ask \emph{does the tradeoff indeed disappear if we rule out optimization issues?}
After all, from a statistical perspective, the robust objective adds information (constraints on the outputs of perturbations) which should intuitively aid generalization, similar to Lasso regression which enforces sparsity~\citep{tibshirani1996regression}.

\begin{table}[t]
  \centering
  \begin{tabular}{c|cc}
     & \textbf{\shortstack{Standard \\training}} & \textbf{\shortstack{Adversarial \\ training}} \\ \hline
    Robust test & $3.5\%$ & $45.8\%$ \\
    Robust train & - &  $100\%$ \\
    \cellcolor{black!15} Standard test & \cellcolor{black!15} $95.2\%$ & \cellcolor{black!15} $87.3\%$ \\
    Standard train & $100\%$ &  $100\%$ \\
  \end{tabular}
  \caption{Train and test accuracies standard and adversarially-trained models on \cifar.
    Both have 100\%
    training accuracy but very different test accuracies. In particular, adversarial training causes worse standard generalization.}
  \label{tab:main-observation}
\end{table}

\paragraph{Contributions.}
We answer the above question negatively by constructing a learning problem with a \emph{convex loss} where adversarial training hurts generalization even when the optimal predictor has both optimal standard and robust accuracy.
Convexity rules out optimization issues, revealing a fundamental statistical explanation for why adversarial training requires more samples to obtain high standard accuracy. Furthermore, we show that we can eliminate the tradeoff in our constructed problem using the recently-proposed robust self-training~\citep{uesato2019are, carmon2019unlabeled, najafi2019robustness, zhai2019adversarially} on additional unlabeled data.

In an attempt to understand how predictive this example is of practice, we subsample \cifar~ and visualize trends in the performance of standard and adversarially trained models with varying training sample sizes.
We observe that the gap between the accuracies of standard and
adversarial training decreases with larger sample size, mirroring the
trends observed in our constructed problem.
Recent results from~\citep{carmon2019unlabeled} show that, similarly to our constructed setting, robust self-training also helps to mitigate the trade-off in~\cifar.

\paragraph{Standard vs. robust generalization.}
Recent work~\citep{schmidt2018adversarially, yin2018rademacher,
khim2018adversarial, montasser2019vc} has focused on the sample complexity of
learning a predictor that has high robust accuracy (robust generalization), a \emph{different objective}.
In contrast, we study the finite sample behavior of adversarially trained
predictors on the standard learning objective (standard
generalization), and show that adversarial training as a particular training procedure could require more samples
to attain high standard accuracy.

\section{Convex learning problem: the staircase}
\label{sec:convex}
\begin{figure*}[t]
  \centering
  \subfigure[Slope $m=1$]{
    \includegraphics[scale=0.3]{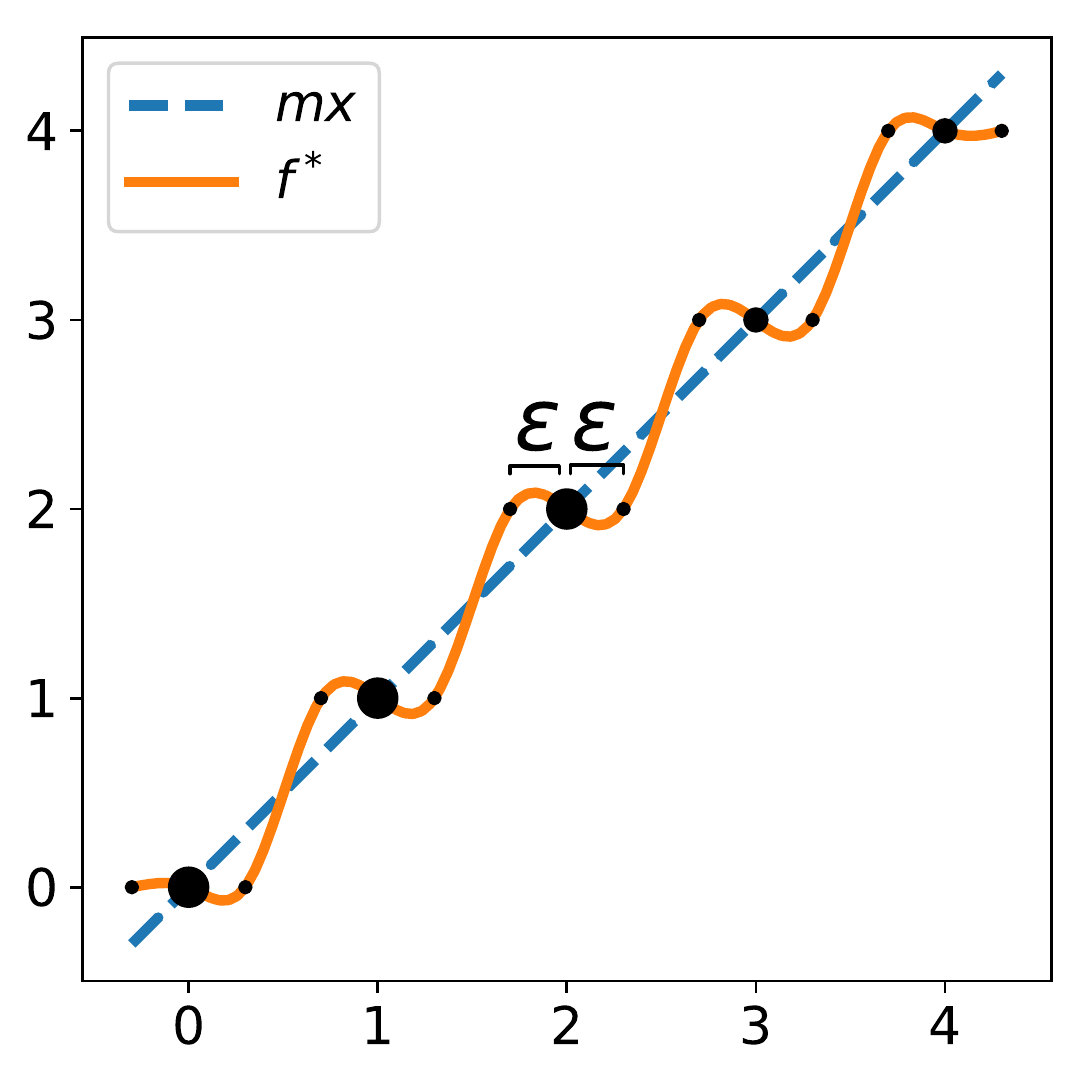}
    \label{fig:schematica}
  }
    \subfigure[Small sample size]{
    \includegraphics[scale=0.3]{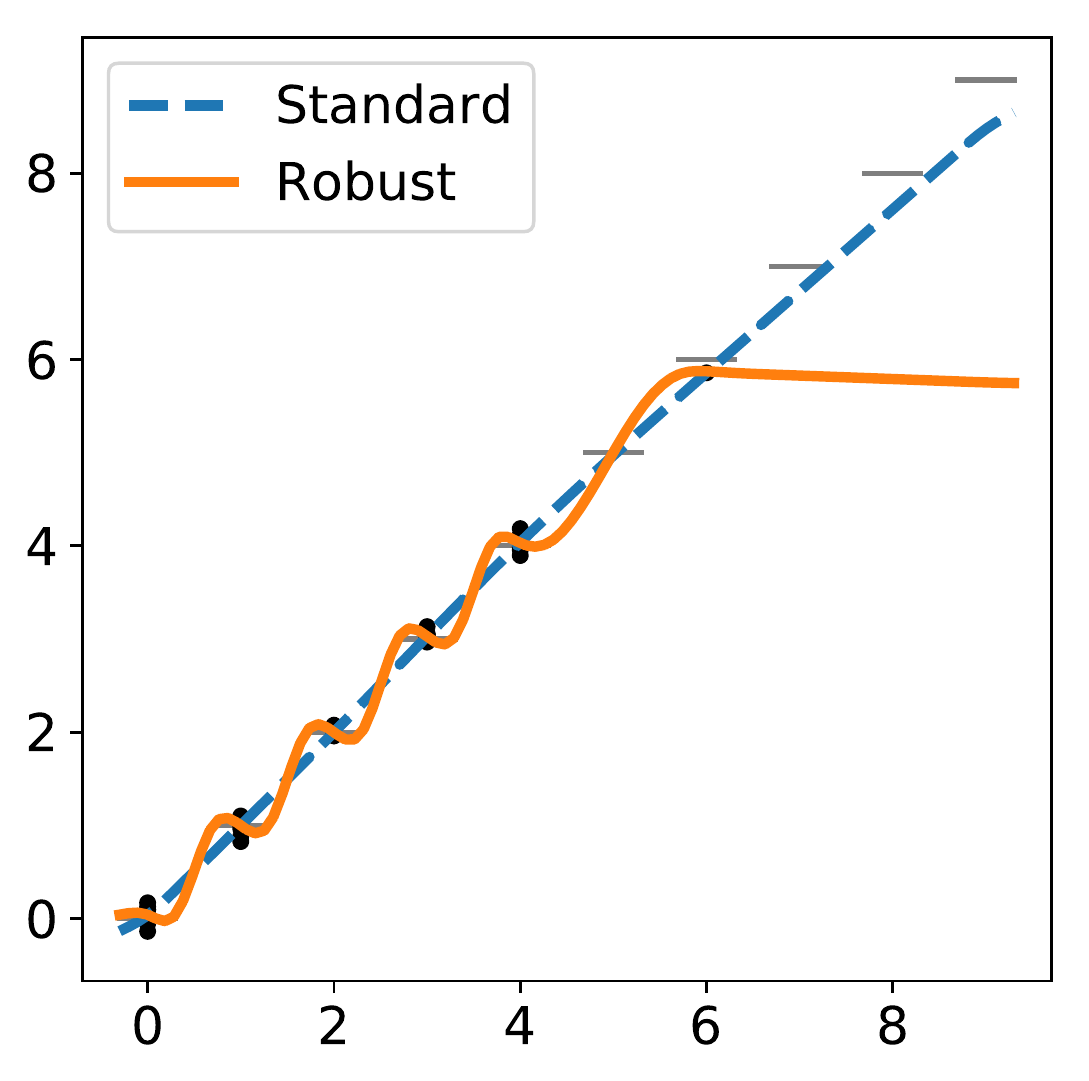}
    \label{fig:tradeoff-small-sample}
  }
    \subfigure[Large sample size]{
      \includegraphics[scale=0.3]{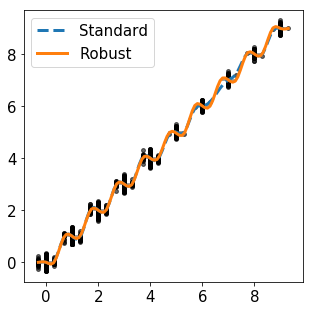}
      \label{fig:tradeoff-large-sample}
  }
  \subfigure[Slope $m=0$]{
    \includegraphics[scale=0.3]{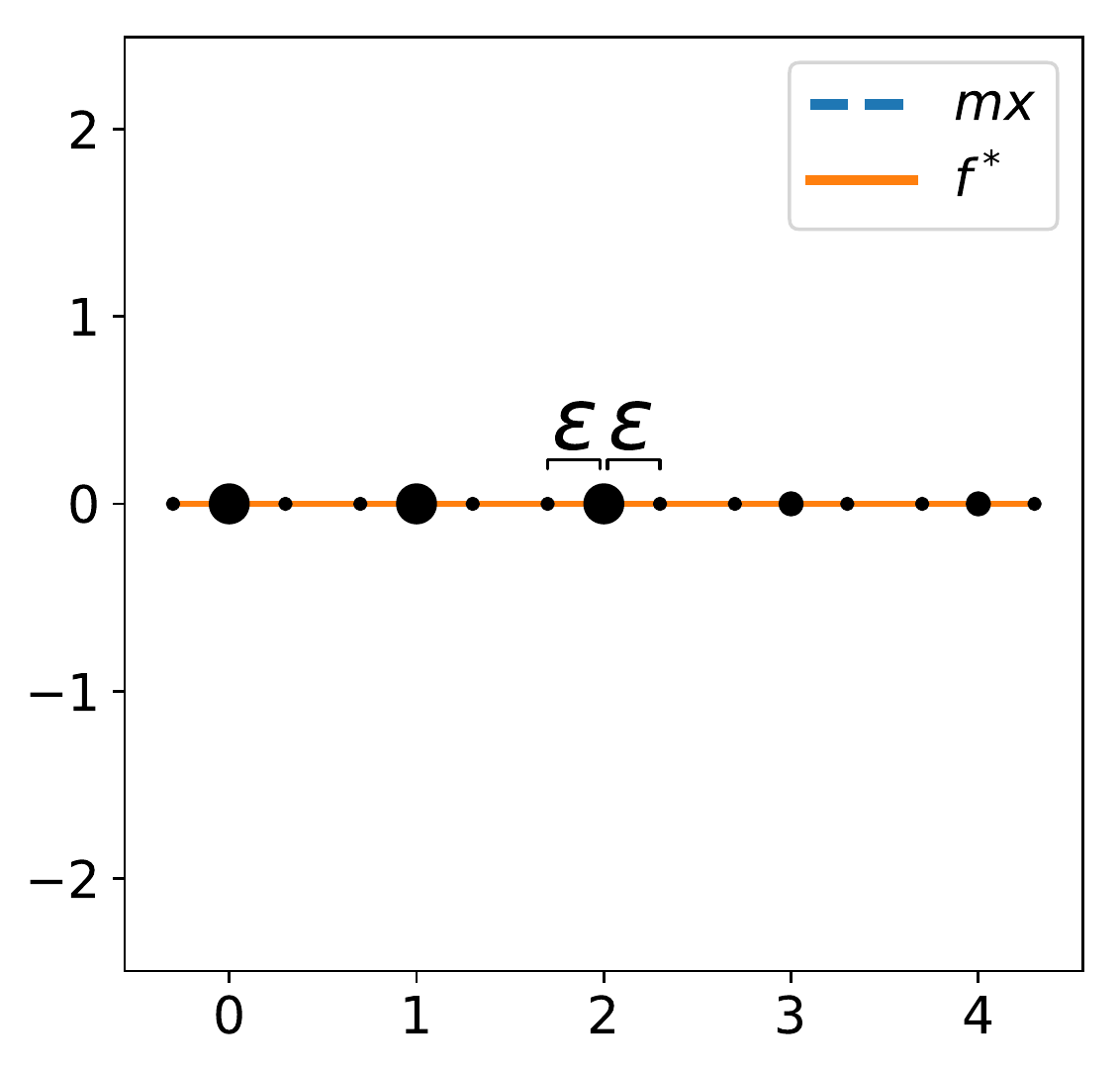}
    \label{fig:schematicb}
  }

    \caption{\textbf{(a):} An illustration of our convex problem with slope $m=1$, with size of the circles proportional to probability under the data distribution. The dashed blue line shows a simple linear predictor that has low test error but not robust to perturbations to nearby low-probability points, while the solid orange line shows the complex optimal predictor $f^\star$ that is both robust and accurate. \textbf{(b):} With small sample size ($n=40$), any robust predictor that fits the sets $B(x)$ is forced to be a staircase that generalizes poorly. \textbf{(c): } With large sample size ($n=25000$), the training set contains all the points from $\xline$ and the robust predictor is close to $\fstar$ by enforcing the right invariances. The standard predictor also has low error, but higher than the robust predictor. 
    \textbf{(d):} An illustration of our convex problem when the slope $m=0$. The optimal predictor $\fstar$ that is robust is a simple linear function. This setting sees no tradeoff for any sample size. }
    \label{fig:splines}
\end{figure*}

We construct a learning problem with the following properties.
First, fitting the majority of the distribution is statistically easy---it
can be done with a \emph{simple} predictor. Second, perturbations of these majority
points are low in probability and require \emph{complex} predictors to be fit.
These two ingredients cause standard estimators to perform better than their adversarially
trained robust counterparts with a few samples. Standard training only
fits the training points, which can be done with a simple estimator that generalizes well;
adversarial training encourages fitting perturbations of the training points
making the estimator complex and generalize poorly.

\subsection{General setup}
We consider mapping $x \in \sX \subset \R$ to $y \in \R$
where $(x,y)$ is a sample from the joint distribution $\Prob$ and conditional densities exist.
We denote by $\Prob_x$ the marginal distribution on $\sX$.
We generate the data as $y = \fstar(x) + \sigma v_i$ where $v_i \overset{\text{i.i.d.}}{\sim} \sN(0, 1)$
and $\fstar:\sX\rightarrow \R$. For an example $(x, y)$, we measure robustness of a predictor with respect to an
\emph{invariance set} $B(x)$ that contains the set of inputs on which the predictor is expected to match the target $y$.

The central premise of this work is that the optimal predictor is robust. In our construction, we let $\fstar$ be robust by enforcing the invariance property (see Appendix~\ref{app-robust-accurate})
\begin{align}
  \label{eqn:invariance}
  f(x) &= f(\tilde{x}), \quad\forall \tilde{x} \in B(x).
\end{align}

Given training data consisting of $n$ i.i.d. samples $(x_i,
y_i) \sim \Prob$, our goal is to learn a predictor $f \in \sF$.
We assume that the hypothesis class $\sF$ contains $\fstar$ and consider the squared loss.
Standard training simply minimizes the empirical risk over the training points.
Robust training seeks to enforce invariance to perturbations of training points by penalizing the worst-case loss over the invariance set $B(x_i)$ with respect to target $y_i$. We consider regularized estimation and obtain the following standard and robust (adversarially trained) estimators for sample size $n$:
\begin{align}
  \label{eqn:stest}
\stest &\in \argmin_{f \in \sF} \sum \limits_{i=1}^n (f(x_i) - y_i)^2 + \lambda \|f\|^2, \\
  \label{eqn:robest}
  \robest &\in \argmin_{f \in \sF} \sum \limits_{i=1}^n \max \limits_{\tilde{x}_i \in B(x_i)} (f(\tilde{x}_i) - y_i)^2  + \lambda \| f \|^2.
\end{align}

We construct a $\Prob$ and $\fstar$ such that both estimators above converge to $\fstar$, but such that the error of the robust estimator $\robest$ is larger than that of $\stest$ for small sample size $n$.
\vspace{2pt}
\subsection{Construction}
\vspace{2pt}
In our construction, we consider linear predictors as ``simple'' predictors that generalize well and \emph{staircase} predictors as ``complex'' predictors that generalize poorly (Figure~\ref{fig:schematica}).
\paragraph{Input distribution.} In order to satisfy the property that a simple predictor fits most of the distribution, we define $\fstar$ to be linear on the set $\xline \subseteq \sX$, where
\begin{align}
  \label{eqn:input}
  \xline &= \{0, 1, 2, \hdots, s-1 \} \nonumber, \\
  \Prob_x(\xline) &= 1 - \delta,
\end{align}
for parameters $\delta\in [0,1]$ and a positive integer $s$.
Any predictor that fits points in $\xline$ will have low (but not optimal) standard error
when $\delta$ is small.

\paragraph{Perturbations.} We now define the perturbations such that that fitting perturbations of the majority of the distribution requires complex predictors.
We can obtain a staircase by flattening out the region around the points in $\xline$ locally (Figure~\ref{fig:schematica}). This motivates our construction where we treat points in $\xline$ as anchor points and the set $\xline^c$ as local perturbations of these points: $x \pm \epsilon$ for $x \in \xline$. This is a simpler version of the commonly studied $\ell_\infty$ perturbations in computer vision. For a point that is not an anchor point, we define $B(x)$ as the invariance set of the closest anchor point $\round{x}$. Formally, for some $\epsilon \in (0, \half)$,
\begin{align}
  \label{eqn:rob-set}
  B(x) &= \{ \round{x}, \round{x} + \epsilon, \round{x} - \epsilon \}.
\end{align}

\paragraph{Output distribution.}
For any point in the support $\sX$,
\begin{align}
  \label{eqn:output}
  \fstar(x) &= m \round{x}, ~\forall x\in\sX,
\end{align}
for some parameter $m$. Setting the slope as $m=1$ makes $\fstar$ resemble a staircase.
Such an $\fstar$ satisfies the invariance property~\refeqn{invariance} that ensures that the optimal predictor for standard error is also robust.
Note that $\fstar(x)=mx$ (a simple linear function) when restricted to $x$ in $\xline$.
Note also that the invariance sets $B(x)$ are disjoint. This is in contrast to the example in~\citep{zhang2019theoretically}, where any invariant function is also globally constant. Our construction allows a non-trivial robust and accurate estimator.

We generate the output by adding Gaussian noise to the optimal predictor $\fstar$, i.e.,
$y = \fstar(x) + \sigma v_i$ where $v_i \overset{\text{i.i.d.}}{\sim} \sN(0, 1)$.

\subsection{Simulations}
\label{sec:simulations}
\begin{figure}[t]
    \centering
    \subfigure[Staircase, $m=1$]{
        \includegraphics[scale=0.33]{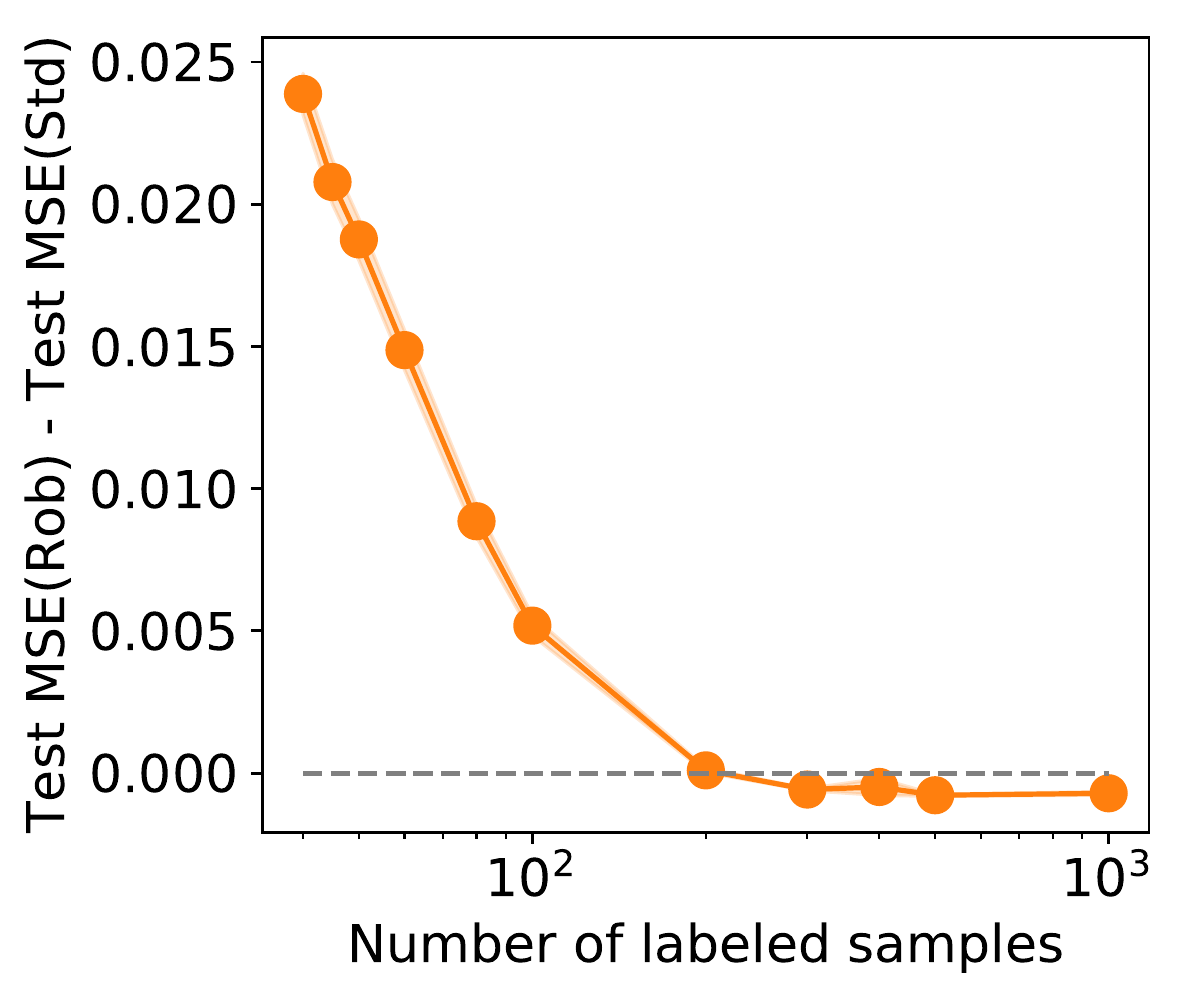}
    }
    \subfigure[$\wrn{40}{2}$ on \cifar]{
      \includegraphics[scale=0.275]{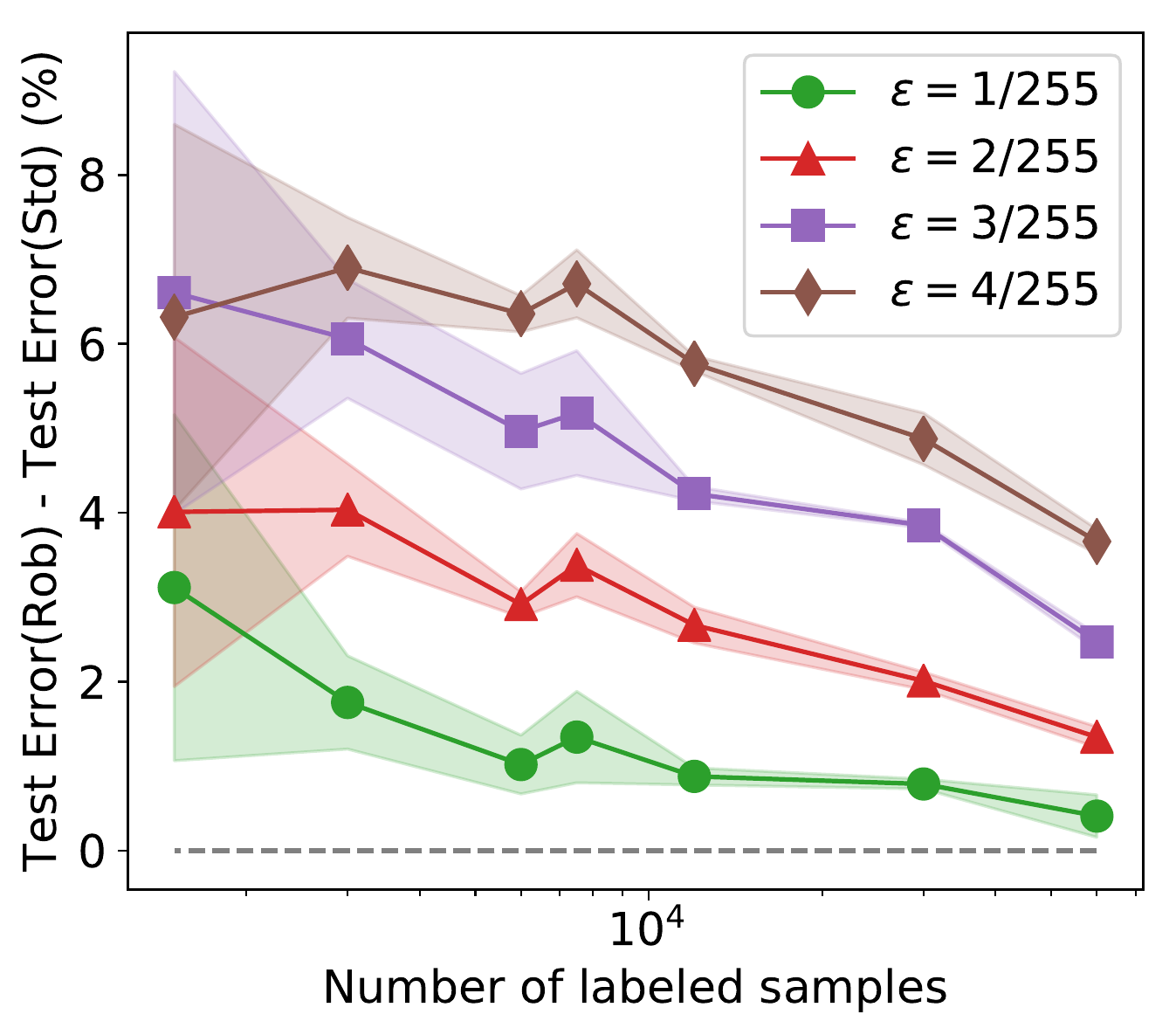}
    }
    \subfigure[Staircase ($m=1$): RST vs. Robust]{
      \includegraphics[scale=0.4]{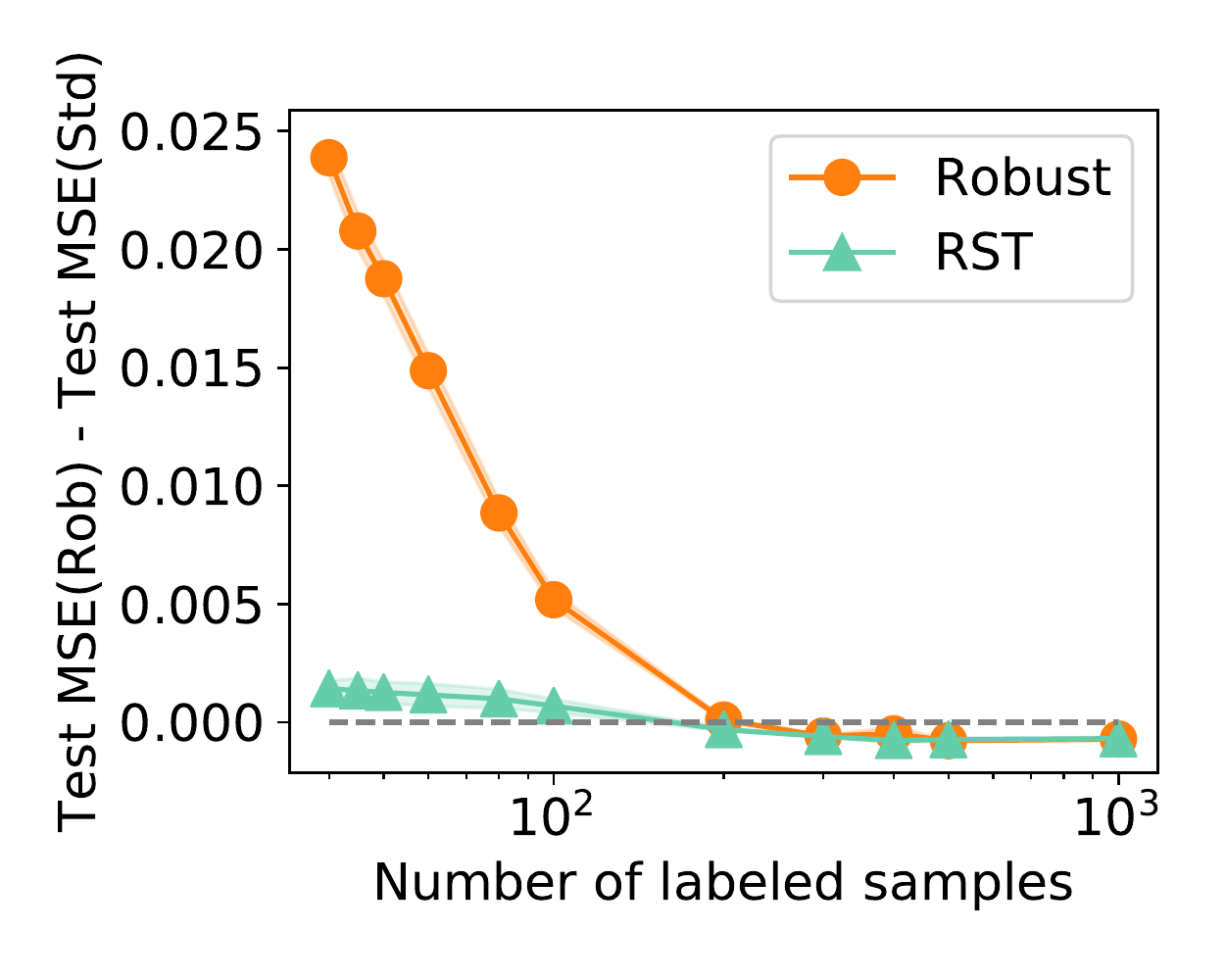}
    }
    \caption{
      Difference between test errors (robust - standard) as a function of the \# of
      training samples $n$. For each $n$, we choose the best regularization parameter $\lambda$ for each of robust and standard training and plot the difference.
      Positive numbers show that the robust estimator has higher MSE than the standard estimator. 
      \textbf{(a)} For the staircase problem with slope $m=1$, we see that for
      small $n$, test loss of the robust estimator is larger.
      As $n$ increases, the gap closes, and eventually the robust estimator has smaller MSE.
      \textbf{(b)} On subsampling \cifar, we see that the gap between test errors (\%) of
      standard and adversarially trained models decreases as the number of
      samples increases, just like the staircase construction in (a). Extrapolating,
      the gap should close as we have more samples.
      \textbf{(c)} Robust self-training (RST), using $1000$ additional unlabeled points, achieves comparable test MSE to standard training (with the same amount of labeled data) and mostly eliminates the tradeoff seen in robust training. The shaded regions represent 1 STD.
    }
    \label{fig:tradeoff}
\end{figure}

We empirically validate the intuition that the staircase problem is sensitive to robust training by simulating training with various sample sizes and comparing the test MSE of the standard and robust estimators~\refeqns{stest}{robest}.
We report final test errors here; trends in generalization gap (difference between train and test error) are nearly identical.
See Appendix~\ref{sec:app-training-details} for more details.

Figure~\ref{fig:tradeoff} shows the difference in test errors of the two estimators. For each sample size $n$, we compare the standard and robust estimators by performing a grid search over regularization parameters $\lambda$ that individually minimize the test MSE of each estimator. With few samples, most training samples are from $\xline$ and standard training learns a simple linear predictor that fits all of $\xline$. On the other hand, robust estimators fit the low probability perturbations $\xline^c$, leading to staircases that generalize poorly. Figure~\ref{fig:tradeoff-small-sample} visualizes the two estimators for small samples. However, as we increase the size of the training set, the training set contains all points from $\xline$, and robust estimators also generalize well despite being more complex. Furthermore, in this regime, robust estimators indeed see the expected ``regularization'' benefit where the robust objective helps fit points in the low probability regions $\xline^c$, even when they are not yet sampled in the training points.
In general, we see that robust training has higher test error with a small sample size, but the difference in the test error of standard and robust estimators decreases as sample size increases, and robust training eventually obtains lower test error.

Another common approach to encoding invariances is data augmentation, where perturbations are \emph{sampled} from $B(x)$ and added to the dataset.
Data augmentation is less demanding than adversarial training which minimizes loss on the \emph{worst-case} point within the invariance set.
We find that for our staircase example, an estimator trained even with the less demanding data augmentation sees a similar tradeoff with small training sets, due to increased complexity of the augmented estimator.

\subsection{Robust self-training mostly eliminates the tradeoff}
\label{sec:rst}

Section~\ref{sec:simulations} shows that the gap between the standard errors of robust and standard estimators decreases as training sample size increases.
Moreover, if we obtained training points spanning $\xline$, then the robust estimator (staircase) would also generalize well and have lower error than the standard estimator.
Thus, a natural strategy to eliminate the tradeoff is to sample more training points.
In fact, we do not need additional labels for the points on $\xline$---a standard trained estimator fits points on $\xline$ with just a few labels, and can be used to generate labels on additional unlabeled points. Recent works have proposed robust self-training (RST) to leverage unlabeled data for robustness~\cite{rosenberg2005semi,carmon2019unlabeled, uesato2019are, najafi2019robustness, zhai2019adversarially}.
RST is a robust variant of the popular self-training algorithm for semi-supervised learning~\citep{rosenberg2005semi}, which uses a standard estimator trained on a few labels to generate psuedo-labels for unlabeled data as described above. See Appendix~\ref{sec:app-rst} for details on RST.

For the staircase problem ($m=1$), RST mostly eliminates the tradeoff and achieves similar test error to standard training (while also being robust, see Appendix~\ref{sec:app-rst-robust}) as shown in Figure~\ref{fig:tradeoff}.

\section{Experiments on \cifar}
In our staircase problem from \refsec{convex}, robust estimators perform worse
on the standard objective because these predictors are more complex, thereby generalizing poorly.
Does this also explain the drop in standard accuracy we see for adversarially trained models on real datasets like \cifar?

We subsample \cifar~ by various amounts to study the effect of sample size on the standard test errors of standard and robust models. To train a robust model, we use the adversarial training procedure from~\citep{madry2018towards} against $\ell_\infty$ perturbations of varying sizes (see Figure~\ref{fig:tradeoff}). The gap in the errors of the standard and adversarially trained models decreases as sample size increases, mirroring the trends in the staircase problem.
Extrapolating the trends, more training data should eliminate the tradeoff in \cifar.
Similarly to the staircase example, \citep{carmon2019unlabeled} showed that robust self-training with additional unlabeled data improves robust accuracy and standard accuracy in \cifar. See Appendix~\ref{sec:app-rst} for more details.

\section{Adversarial training can also help}
\label{sec:robust_helps}

\begin{figure}[t]
    \centering
    \subfigure[Staircase, $m=0$]{
      \includegraphics[scale=0.35]{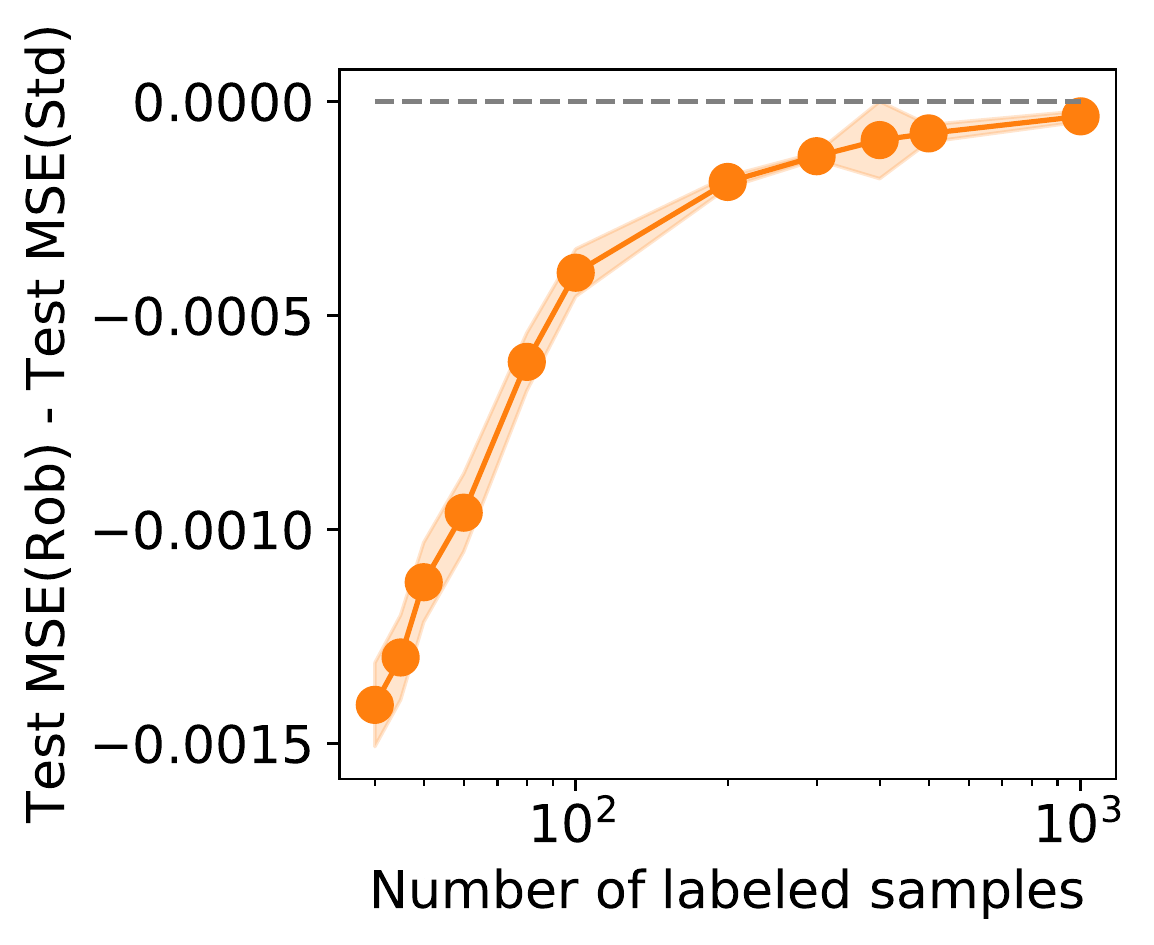}
  }
    \subfigure[Small CNN on \mnist]{
      \includegraphics[scale=0.34]{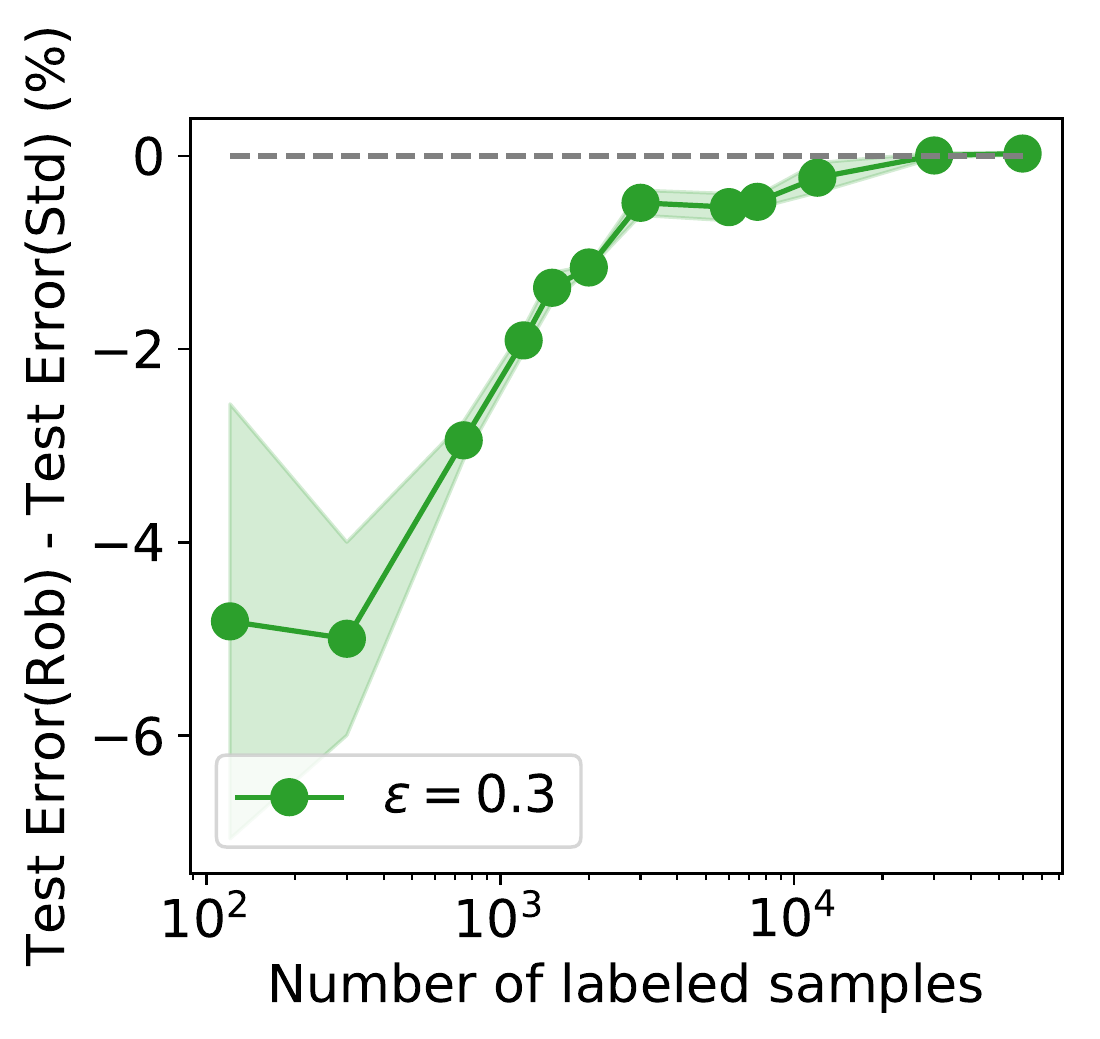}
    }
    \caption{
      Difference between test errors (robust - standard) as a function of the \# of
      training samples $n$. For each $n$, we choose the best regularization parameter $\lambda$ for each of robust and standard training and take the difference.
      Negative numbers mean that robust training has a lower test MSE than standard training.
      \textbf{(a)} In the staircase problem with slope $m=0$, the robust estimator consistently outperforms the standard estimator, showing a regularization benefit.
    \textbf{(b)} On \mnist, the adversarially trained model has lower test error (\%) than the standard model. The difference in test errors is largest for small sample sizes and closes with more training samples. Shaded regions represent 1 STD.}
    \label{fig:tradeon}
\end{figure}
One of the key ingredients that causes the tradeoff in the staircase problem is the complexity of robust predictors. If we change our construction such that robust predictors are also simple, we see that adversarial training instead offers a regularization benefit.
When $m=0$, the optimal predictor (which is robust) is linear (Figure~\ref{fig:schematicb}).
We find that adversarial training has lower standard error by enforcing invariance on $B(x)$ making the robust estimator less sensitive to target noise (Figure~\ref{fig:tradeon-small-sample}).

Similarly, on \mnist, the adversarially trained model has lower test error than standard trained model. As we increase the sample size, both standard and adversarially trained models converge to obtain same small test error. We remark that our observation on \mnist is contrary to that reported in \citep{tsipras2019robustness}, due to a different initialization that led to better optimization (see Appendix Section~\ref{sec:app-init-mnist}). 

\section{Conclusion}
\label{sec:conclusion}
In this work, we shed some light on the counter-intuitive phenomenon where
enforcing invariance respected by the optimal function could actually degrade performance.
Being invariant could require complex predictors and consequently more samples to generalize well. Our experiments support that the tradeoff between robustness and accuracy observed in practice is indeed due to insufficient samples and additional unlabeled data is sufficient to mitigate this tradeoff.

\newpage
\section*{Acknowledgements}
We are grateful to Tengyu Ma for several helpful discussions. This work was funded by an Open Philanthropy Project Award and NSF Frontier Award Grant no. 1805310. AR was supported by Google Fellowship and Open Philanthropy AI Fellowship. FY was supported by the Institute for Theoretical Studies ETH Zurich and the Dr. Max R\"ossler and the Walter Haefner Foundation. FY and JCD were supported by the Office of Naval Research Young Investigator Award N00014-19-1-2288.

\bibliography{refdb/all}

\begin{thebibliography}{19}
\providecommand{\natexlab}[1]{#1}
\providecommand{\url}[1]{\texttt{#1}}
\expandafter\ifx\csname urlstyle\endcsname\relax
  \providecommand{\doi}[1]{doi: #1}\else
  \providecommand{\doi}{doi: \begingroup \urlstyle{rm}\Url}\fi

\bibitem[Carmon et~al.(2019)Carmon, Raghunathan, Schmidt, Liang, and
  Duchi]{carmon2019unlabeled}
Y.~Carmon, A.~Raghunathan, L.~Schmidt, P.~Liang, and J.~C. Duchi.
\newblock Unlabeled data improves adversarial robustness.
\newblock \emph{arXiv preprint arXiv:1905.13736}, 2019.

\bibitem[Diamond and Boyd(2016)]{diamond2016cvxpy}
S.~Diamond and S.~Boyd.
\newblock {CVXPY}: A {P}ython-embedded modeling language for convex
  optimization.
\newblock \emph{Journal of Machine Learning Research (JMLR)}, 17\penalty0
  (83):\penalty0 1--5, 2016.

\bibitem[Friedman et~al.(2001 2001)Friedman, Hastie, and
  Tibshirani]{friedman2001elements}
J.~Friedman, T.~Hastie, and R.~Tibshirani.
\newblock \emph{The elements of statistical learning}, volume~1.
\newblock Springer series in statistics New York, NY, USA: Springer series in
  statistics New York, NY, USA:, 2001 2001.

\bibitem[Goodfellow et~al.(2015)Goodfellow, Shlens, and
  Szegedy]{goodfellow2015explaining}
I.~J. Goodfellow, J.~Shlens, and C.~Szegedy.
\newblock Explaining and harnessing adversarial examples.
\newblock In \emph{International Conference on Learning Representations
  (ICLR)}, 2015.

\bibitem[Khim and Loh(2018)]{khim2018adversarial}
J.~Khim and P.~Loh.
\newblock Adversarial risk bounds for binary classification via function
  transformation.
\newblock \emph{arXiv preprint arXiv:1810.09519}, 2018.

\bibitem[Madry et~al.(2018)Madry, Makelov, Schmidt, Tsipras, and
  Vladu]{madry2018towards}
A.~Madry, A.~Makelov, L.~Schmidt, D.~Tsipras, and A.~Vladu.
\newblock Towards deep learning models resistant to adversarial attacks.
\newblock In \emph{International Conference on Learning Representations
  (ICLR)}, 2018.

\bibitem[Montasser et~al.(2019)Montasser, Hanneke, and Srebro]{montasser2019vc}
O.~Montasser, S.~Hanneke, and N.~Srebro.
\newblock {VC} classes are adversarially robustly learnable, but only
  improperly.
\newblock \emph{arXiv preprint arXiv:1902.04217}, 2019.

\bibitem[Najafi et~al.(2019)Najafi, Maeda, Koyama, and
  Miyato]{najafi2019robustness}
A.~Najafi, S.~Maeda, M.~Koyama, and T.~Miyato.
\newblock Robustness to adversarial perturbations in learning from incomplete
  data.
\newblock \emph{arXiv preprint arXiv:1905.13021}, 2019.

\bibitem[Nakkiran(2019)]{nakkiran2019adversarial}
P.~Nakkiran.
\newblock Adversarial robustness may be at odds with simplicity.
\newblock \emph{arXiv preprint arXiv:1901.00532}, 2019.

\bibitem[Rosenberg et~al.(2005)Rosenberg, Hebert, and
  Schneiderman]{rosenberg2005semi}
C.~Rosenberg, M.~Hebert, and H.~Schneiderman.
\newblock Semi-supervised self-training of object detection models.
\newblock In \emph{Proceedings of the Seventh IEEE Workshops on Application of
  Computer Vision}, 2005.

\bibitem[Schmidt et~al.(2018)Schmidt, Santurkar, Tsipras, Talwar, and
  Madry]{schmidt2018adversarially}
L.~Schmidt, S.~Santurkar, D.~Tsipras, K.~Talwar, and A.~Madry.
\newblock Adversarially robust generalization requires more data.
\newblock In \emph{Advances in Neural Information Processing Systems
  (NeurIPS)}, pages 5014--5026, 2018.

\bibitem[Szegedy et~al.(2014)Szegedy, Zaremba, Sutskever, Bruna, Erhan,
  Goodfellow, and Fergus]{szegedy2014intriguing}
C.~Szegedy, W.~Zaremba, I.~Sutskever, J.~Bruna, D.~Erhan, I.~Goodfellow, and
  R.~Fergus.
\newblock Intriguing properties of neural networks.
\newblock In \emph{International Conference on Learning Representations
  (ICLR)}, 2014.

\bibitem[Tibshirani(1996)]{tibshirani1996regression}
R.~Tibshirani.
\newblock Regression shrinkage and selection via the lasso.
\newblock \emph{Journal of the Royal Statistical Society: Series B
  (Methodological)}, 58\penalty0 (1):\penalty0 267--288, 1996.

\bibitem[Tsipras et~al.(2019)Tsipras, Santurkar, Engstrom, Turner, and
  Madry]{tsipras2019robustness}
D.~Tsipras, S.~Santurkar, L.~Engstrom, A.~Turner, and A.~Madry.
\newblock Robustness may be at odds with accuracy.
\newblock In \emph{International Conference on Learning Representations
  (ICLR)}, 2019.

\bibitem[Uesato et~al.(2019)Uesato, Alayrac, Huang, Stanforth, Fawzi, and
  Kohli]{uesato2019are}
J.~Uesato, J.~Alayrac, P.~Huang, R.~Stanforth, A.~Fawzi, and P.~Kohli.
\newblock Are labels required for improving adversarial robustness?
\newblock \emph{arXiv preprint arXiv:1905.13725}, 2019.

\bibitem[Yin et~al.(2018)Yin, Ramchandran, and Bartlett]{yin2018rademacher}
D.~Yin, K.~Ramchandran, and P.~Bartlett.
\newblock Rademacher complexity for adversarially robust generalization.
\newblock \emph{arXiv preprint arXiv:1810.11914}, 2018.

\bibitem[Zagoruyko and Komodakis(2016)]{zagoruyko2016wide}
S.~Zagoruyko and N.~Komodakis.
\newblock Wide residual networks.
\newblock In \emph{British Machine Vision Conference}, 2016.

\bibitem[Zhai et~al.(2019)Zhai, Cai, He, Dan, He, Hopcroft, and
  Wang]{zhai2019adversarially}
R.~Zhai, T.~Cai, D.~He, C.~Dan, K.~He, J.~Hopcroft, and L.~Wang.
\newblock Adversarially robust generalization just requires more unlabeled
  data.
\newblock \emph{arXiv preprint arXiv:1906.00555}, 2019.

\bibitem[Zhang et~al.(2019)Zhang, Yu, Jiao, Xing, Ghaoui, and
  Jordan]{zhang2019theoretically}
H.~Zhang, Y.~Yu, J.~Jiao, E.~P. Xing, L.~E. Ghaoui, and M.~I. Jordan.
\newblock Theoretically principled trade-off between robustness and accuracy.
\newblock In \emph{International Conference on Machine Learning (ICML)}, 2019.

\end{thebibliography}
\bibliographystyle{abbrvnat}
\newpage
\appendix
\section{Consistency of robust and standard estimators}
\label{app-robust-accurate}
We show that the invariance condition (restated, \refeqn{regression-assumption}) is a sufficient condition for the minimizers of the standard and robust objectives under $\Prob$ in the infinite data limit to be the same.
\begin{align}
  \label{eqn:regression-assumption}
  \fstar(x) = \fstar(\tilde{x})~~\forall \tilde{x} \in B(x), 
\end{align}
for all $x \in \sX$.

Recall that $y = \fstar(x) + \sigma v_i$ where $v_i \overset{\text{i.i.d.}}{\sim} \sN(0, 1)$,
with $\fstar(x) = \E[y \mid x]$. 
Therefore, if $f^\star$ is in the hypothesis class $\sF$, then $\fstar$ minimizes the standard objective for the square loss. 

If both $\stest$~\refeqn{stest} and $\robest$~\refeqn{robest} converge to the same Bayes optimal $\truef$ as $n\rightarrow \infty$, we say that the two estimators $\stest$ and $\robest$ are \emph{consistent}. In this section, we show that the invariance condition~\refeqn{regression-assumption} implies consistency of $\robest$ and $\stest$. 

Intuitively, from \refeqn{regression-assumption}, since $\truef$ is invariant for all $x$ in $B(x)$, the maximum over $B(x)$ in the robust objective is achieved by the unperturbed input $x$ (and also achieved by any other element of $B(x)$). Hence the standard and robust loss of $\fstar$ are equal. For any other predictor, the robust loss upper bounds the standard loss, which in turn is an upper bound on the standard loss of $\truef$ (since $\truef$ is Bayes optimal). Therefore $\truef$ also obtains optimal robust loss and $\stest$ and $\robest$ are consistent and converge to $\truef$ with infinite data.

Formally, let $\ell$ be the square loss function, and the population loss be $\E_{(x,y)\sim\Prob}[\ell(f(x),y)]$.
In this section, all expectations are taken over the joint distribution $\Prob$.
\begin{theorem}
    \label{thm:regression}
    (Regression) Consider the minimizer of the standard population squared loss,
    $f^*=\argmin_f \E[\ell(f(x), y)]$ where $\ell(f(x), y) = (f(x)-y)^2$.
    Assuming \refeqn{regression-assumption} holds, we have that for any $f$, $\E[\max_{\tilde{x}\in B(x)}\ell(f(x),y)] \geq \E[\max_{\tilde{x}\in B(x)}\ell(f^*(x),y)]$, such that $f^*$ is also optimal for the robust population squared loss.
\end{theorem}
\begin{proof}
Note that the optimal standard model is the Bayes estimator, such that $f^*(x) = \E[y \mid x]$.
Then by condition \refeqn{regression-assumption}, $f^*(\tilde{x}) = \E[y\mid \tilde{x}] = \E[y\mid x] = f^*(x)$ for all $\tilde{x}\in B(x)$.
Thus the robust objective for $f^*$ is
\begin{align*}
    \E[\max_{\tilde{x}\in B(x)} \ell(f^*(x), y)] &=  \E\left[\max_{\tilde{x}\in B(x)} (\E[y\mid \tilde{x}] - y)^2\right]\\
    &=  \E\left[(\E[y\mid x] - y)^2\right]\\
    &= \E[\ell(f^*(x),y)]
\end{align*}
where the first equality follows because $f^*$ is the Bayes estimator and the second equality is from \refeqn{regression-assumption}.
Noting that for any classifier $f$, $\E[\max_{\tilde{x}\in B(x)} \ell(f(x), y)] \geq \E[\ell(f(x), y)] \geq \E[\ell(f^\star(x), y)]$, the theorem statement follows.
\end{proof}

For the classification case, consistency requires label invariance,
which is that
\begin{align}
\label{eqn:assumption-again}
    \argmax_y ~p(y \mid x) = \argmax_y ~p(y \mid \tilde{x}) ~~ \forall \tilde{x} \in B(x),
\end{align}
such that the adversary cannot change the label that achieves the maximum but can perturb the distribution.

The optimal standard classifier here is the Bayes optimal classifier $\fstar_c = \argmax_y p(y\mid x)$.
Assuming that $\fstar_c=\argmax_y p(y\mid x)$ is in $\sF$, then consistency follows by essentially the same argument as in the regression case.

\begin{theorem}
    \label{thm:classification}
    (Classification) Consider the minimizer of the standard population 0-1 loss,
    $\fstar_c=\argmin_f \E[\ell(f(x), y)]$ where $\ell(f(x), y) = \mathbf{1}\{\argmax_j f(x)_j = y\}$.
    Assuming \refeqn{assumption-again} holds, we have that for any $f$, $\E[\max_{\tilde{x}\in B(x)}\ell(f(x), y)] \geq \E[\max_{\tilde{x}\in B(x)}\ell(\fstar_c(x),y)]$, such that $\fstar_c$ is also optimal for the robust population 0-1 loss.
\end{theorem}
\begin{proof}
    Replacing $\fstar$ with $\fstar_c$ and $\ell(f(x),y)$ with the zero-one loss $\mathbf{1}\{\argmax_j f(x)_j = y\}$ in the proof of \refthm{regression} gives the result.
\end{proof}

In our staircase problem, from~\refeqn{invariance}, we assume that the target $y$ is generated as follows: $y = \fstar(x) + \sigma v_i$ where $v_i \overset{\text{i.i.d.}}{\sim} \sN(0, 1)$, we see that the points within an invariance sets $B(x)$ have the same target distribution (target distribution invariance).
\begin{align}
  \label{eqn:assumption-again-strict}
  \fstar(x) &= \fstar(\tilde{x})~~\forall \tilde{x} \in B(x) \\
  \implies p(y \mid x) &= p(y \mid \tilde{x}) \enspace\enspace \forall \tilde{x} \in B(x), 
\end{align}
for all $x\in\X$.

The target invariance condition above implies consistency in both the
regression and classification case.

\section{Convex staircase example}
\subsection{Data distribution}
\paragraph{Distribution of $\sX$.} We focus on a 1-dimensional regression case. Let $s$ be the total number of ``stairs'' in the staircase problem. Let $s_0\leq s$ be the number of stairs that have a large weight in the data distribution.
Define $\delta\in[0,1]$ to be the probability of sampling a perturbation point, i.e. $x\in \xline^c$, which we will choose to be close to zero.
The size of the perturbations is $\epsilon \in [0, \half)$, which is bounded by $\half$ so that $\round{x \pm \epsilon} = x$, for any $x \in \xline$.
The standard deviation of the noise in the targets is $\sigma>0$.
Finally, $m\in[0,1]$ is a parameter controlling the slope of the points in $\xline$.

Let $w\in \Delta_s$ be a distribution over $\xline$ where $\Delta_s$ is the probability simplex of dimension $s$.
We define the data distribution with the following generative process for one sample $x$.
First, sample a point $i$ from $\xline$ according to the categorical distribution described by $w$, such that $i\sim \text{Categorical}(w)$.
Second, sample $x$ by perturbing $i$ with probability $\delta$ such that
\[
x =
\begin{cases}
    i & \text{w.p. } 1-\delta\\
    i-\epsilon & \text{w.p. } \delta/2\\
    i+\epsilon & \text{w.p. } \delta/2.
\end{cases}
\]
Note that this is just a formalization of the distribution described in \refsec{convex}.
The sampled $x$ is in $\xline$ with probability $1-\delta$ and $\xline^c$ with probability $\delta$, where we choose $\delta$ to be small.

In addition, in order to exaggerate the difference between robust and standard estimators for small sample sizes,
we set $w$ such that the first $s_0$ stairs have the majority of probability mass.
To achieve this, we set the unnormalized probabilities of $w$ as
\[
\hat{w}_j =
\begin{cases}
    1/s_0 & j < s_0\\
    0.01 & j\geq s_0
\end{cases}
\]
and define $w$ by normalizing $w=\hat{w}/\sum_j\hat{w}_j$.
For our examples, we fix $s_0=5$.
In general, even though we can increase $s$ to create versions of our example with more stairs, $s_0$ is fixed to highlight the bad extrapolation behavior of the robust estimator.

\paragraph{Distribution of $\sY$.}
We define the target distribution as $(Y\mid X=x)\sim\mathcal{N}(m\lfloor x\rceil, \sigma^2)$, where $\lfloor x\rceil$ rounds $x$ to the nearest integer.
The invariance sets are $B(x) = \{\lfloor x\rceil-\epsilon, \lfloor x \rceil, \lfloor x\rceil+\epsilon\}$.
We define the distribution such that for any $x$, all points in $B(x)$ have the same mean target value $m \lfloor x \rceil$.
See Figure~\ref{fig:splines} for an illustration.

Note that $B(x)$ is defined such that (\refeqn{assumption-again-strict}) holds, since for any $x_1,x_2\in B(x)$, $\lfloor x_1 \rceil= \lfloor x_2\rceil$ and thus $p(y\mid x_1)=p(y\mid x_2)$. The conditional distributions are defined since $p(\tilde{x})>0$ for any $\tilde{x}\in B(x)$.

\subsection{Model}
Our hypothesis class is the family of cubic B-splines as defined in~\citep{friedman2001elements}.
Cubic B-splines are piecewise cubic functions, where the endpoints of each cubic function are called the knots.
In our example, we fix the knots to be
$\mathbf{\tau}=[-\epsilon, 0, \epsilon, \dots, (s-1)-\epsilon, s-1, (s-1)+\epsilon]$,
which places a knot on every point on the support of $\sX$. This ensures that the family is expressive enough to include $f^\star$, which is any function in $\sF$ which satisfies $\truef(x) = m \lfloor x \rceil$ for all $x$ in $\sX$.
Cubic B-splines can be viewed as a kernel method with kernel feature map $\Phi:\X\rightarrow \RN^{3s+2}$, where $s$ is the number of stairs in the example.

For some regularization parameter $\lambda \geq 0$ we optimize with the penalized smoothing spline loss function over parameters $\theta$,
\begin{align}
  \ell(f_\theta(x), y) =&(y - f_\theta(x))^2 + \lambda\int(f_\theta''(t))^2dt \\
  &= (y - \Phi(x)^T\theta)^2 + \lambda\theta^T\Omega\theta,
\end{align}
where $\Omega_{i,j}=\int \Phi''(t)_i\Phi''(t)_jdt$ measures smoothness in terms of the second derivative.With respect to the regularized objectives (\ref{eqn:stest}) and (\ref{eqn:robest}), the norm regularizer is $\|f\|^2 = \theta^T\Omega\theta$.

We implement the optimization of the standard and robust objectives using the basis described in~\citep{friedman2001elements}.
The regularization penalty matrix $\Omega$ computes second-order finite differences of the parameters $\theta$.
Suppose we have $n$ samples of training inputs $X=\{x_1,\dots, x_n\}$ and targets $\by=\{y_1,\dots, y_n\}$ drawn from $\Prob$ .
The standard spline objective solves the linear system
\begin{align*}
    \hat{\theta}_{\std} = (\Phi(X)^T\Phi(X) + \lambda\Omega)^{-1}\Phi(X)^T\by,
\end{align*}
where the $i$-th row of $\Phi(X)\in \RN^{n\times (3s+2)}$ is $\Phi(x_i)$.
The standard estimator is then $\stest(x)=\Phi(x)^T\hat{\theta}_{\std}$.
We solve the robust objective directly as a pointwise maximum of squared losses over the invariance sets (which is still convex) using CVXPY~\cite{diamond2016cvxpy}.

\subsection{Role of different parameters}

To construct an example where robustness hurts generalization, the main parameters needed are that the slope $m$ is large and that the probability $\delta$ of drawing samples from perturbation points $\xline^c$ is small.
When slope $m$ is large, the complexity of the true function increases such that good generalization requires more samples.
A small $\delta$ ensures that a low-norm linear solution has low test error.
This example is insensitive to whether there is label noise, meaning that $\sigma=0$ is sufficient to observe that robustness hurts generalization.

If $m\approx 0$, then the complexity of the true function is low and we observe that robustness helps generalization.
In contrast, this example relies on the fact that there is label noise ($\sigma > 0$) so that the noise-cancelling effect of robust training improves generalization.
In the absence of noise, robustness neither hurts nor helps generalization since both the robust and standard estimators converge to the true function ($f^*(x)=0$) with only one sample.

\begin{figure}[t!]
  \centering
    \subfigure[Small sample ($n=40$)]{\includegraphics[scale=0.3]{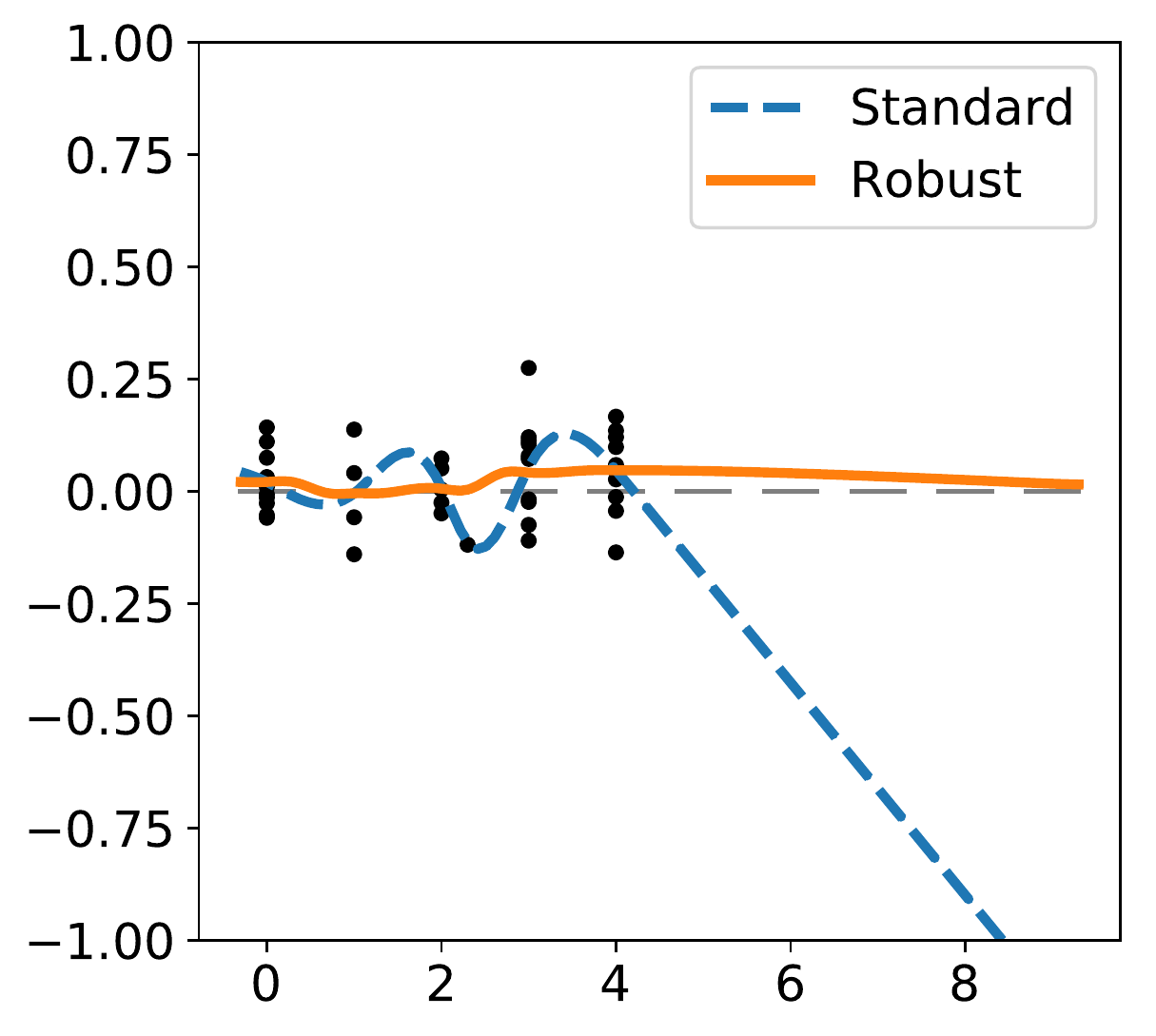}\label{fig:tradeon-small-sample}}
    \subfigure[Large sample ($n=25000$)]{
    \includegraphics[scale=0.3]{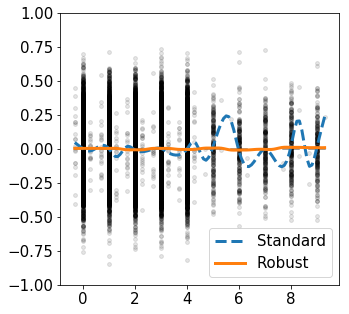}
    \label{fig:tradeon-large-sample}
  }
  \caption{\textbf{Left:} With small samples, the standard solution may overfit to noise, while adversarial training has a noise cancelling effect. \textbf{Right:} With large samples, both the robust and standard predictors have low test error, but the standard predictor is still more susceptible to noise.}
  \label{fig:splines-2}
\end{figure}

\subsection{Plots of other values}
\label{app-extraplots}

We show plots for a variety of quantities against number of samples $n$.
For each $n$, we pick the best regularization parameter $\lambda$ with respect to standard test MSE individually for robust and standard training.
in the $m=1$ (robustness hurts) and $m=0$ (robustness helps) cases, with all the same parameters as before. In both cases, the test MSE and generalization gap (difference between training MSE and test MSE) are almost identical due to robust and standard training having similar training errors. In the $m=1$ case where robustness hurts (Figure~\ref{fig:tradeoff-othervalues}), robust training finds higher norm estimators for all sample sizes. With enough samples, standard training begins to increase the norm of its solution as it starts to converge to the true function (which is complex) and the robust train MSE starts to drop accordingly.

In the $m=0$ case where robustness helps (Figure~\ref{fig:tradeon-othervalues}), the optimal predictor is the line $f(x)=0$, which has 0 norm. The robust estimator has consistently low norm. With small sample size, the standard estimator has low norm but has high test MSE. This happens when the standard estimator is close to linear (has low norm), but the estimator has the wrong slope, causing high test MSE. However, in the infinite data limit, both standard and robust estimators converge to the optimal solution.

\section{Robust self-training algorithm}
\label{sec:app-rst}

\begin{figure}[t]
    \centering
      \subfigure[Robust training vs. RST]{
        \includegraphics[scale=0.3]{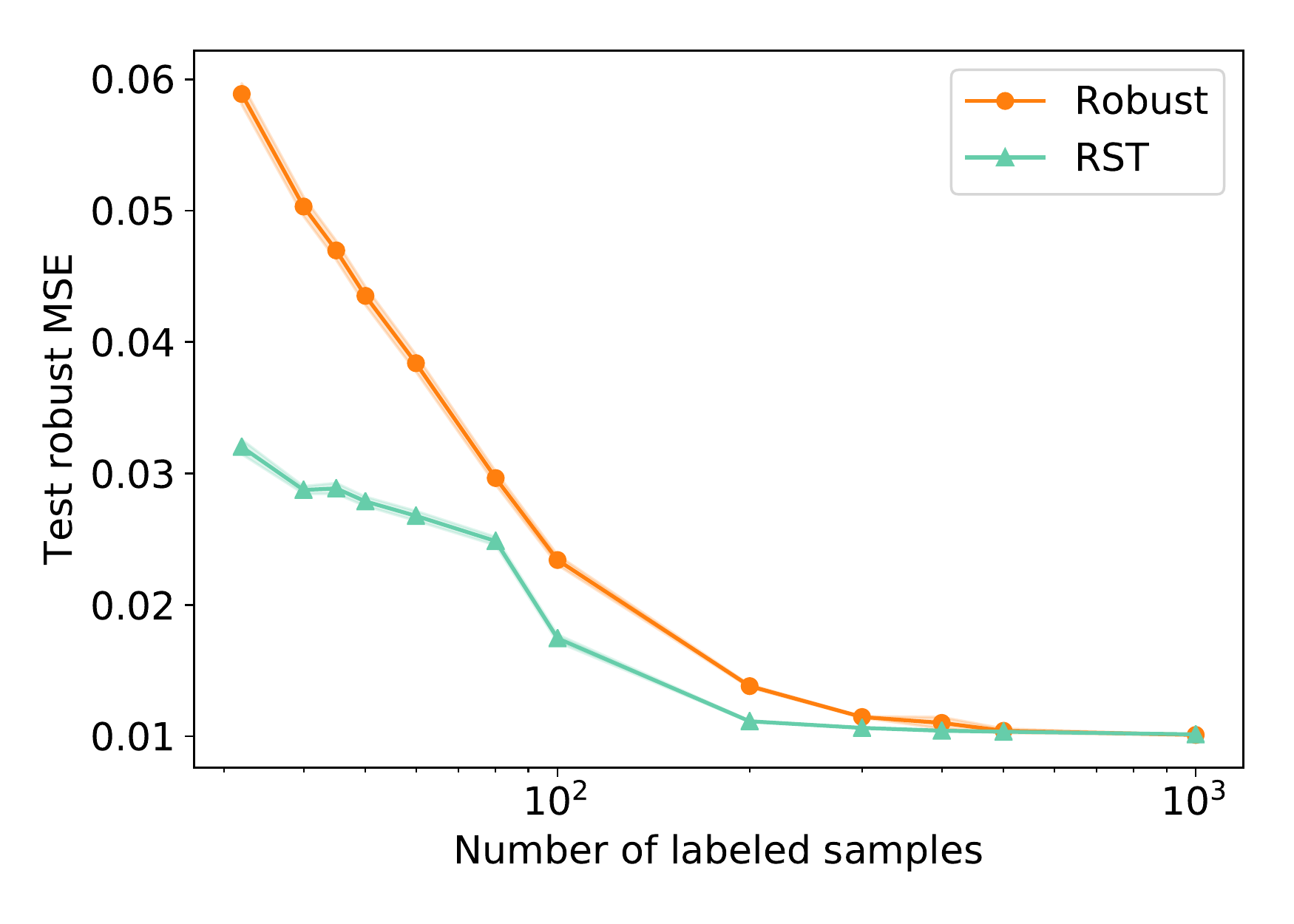}
        \label{fig:compare-rst-2-a}
      }
      \subfigure[Standard training vs. RST]{
        \includegraphics[scale=0.3]{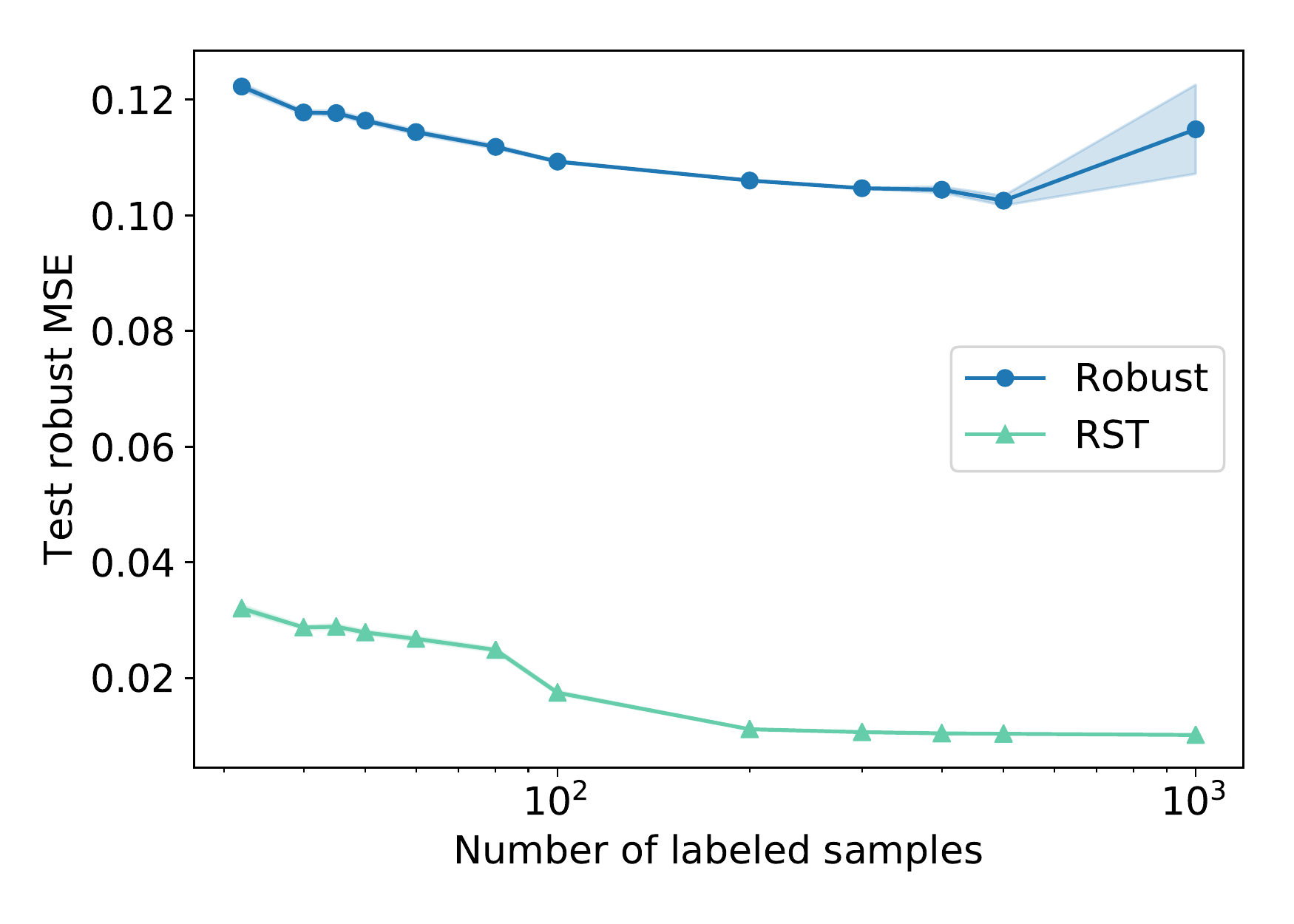}
        \label{fig:compare-rst-2-b}
      }
      \caption{Robust self-training (RST) improves test robust MSE (not just standard test MSE) over both standard and robust training. For each $n$, the regularization parameter $\lambda$ is chosen with respect to the best test MSE over a grid search for each of robust, RST, and standard training. \textbf{(a)} shows that robust self-training improves robust error over robust training. \textbf{(b)} confirms that robust self-training also improves robust test error over standard training.}
    \label{fig:compare-rst-2}
\end{figure}

We describe the robust self-training procedure, which performs robust training on a dataset augmented with unlabeled data.
The targets for the unlabeled data are generated from a standard estimator trained on the labeled training data.
Since the standard estimator has good standard generalization, the generated targets for the unlabeled data have low error on expectation.
Robust training on the augmented dataset seeks to improve both the standard and robust test error of robust training (over just the labeled training data).
Intuitively, robust self-training achieves these gains by mimicking the standard estimator on more of the data distribution (by using unlabeled data) while also optimizing the robust objective.

In robust self-training, we are given $n$ samples of training inputs $X=\{x_1,\dots, x_n\}$ and targets $\by=\{y_1,\dots, y_n\}$ drawn from $\Prob$ .
Suppose that we have additional $m$ unlabeled samples $X_u$ drawn from $\Prob_x$.
Robust self-training uses the following steps for a given regularization $\lambda$:
\begin{enumerate}
    \item Compute the standard estimator $\stest$~\refeqn{stest} on the labeled data $(X,~\by)$ with regularization parameter $\lambda$.
    \item Generate pseudo-targets $\by_u = \stest(X_u)$ by evaluating the standard estimator obtained above on the unlabeled data $X_u$.
    \item Construct an augmented dataset $X_{\text{aug}}=X\cup X_u$, $\by_{\text{aug}}=\by~\cup~\by_u$.
    \item Return a robust estimator $\robest$~\refeqn{robest} with the augmented dataset $(X_{\text{aug}},~\by_{\text{aug}})$ as training data.
\end{enumerate}

\subsection{Results on \cifar}
We present relevant results from the recent work of~\citep{carmon2019unlabeled} on robust self-training applied on \cifar~ augmented with unlabeled data in \reftab{rst-cifar}. The procedure employed in~\citep{carmon2019unlabeled} is identical to the procedure describe above, using a modified version of adversarial training (TRADES)~\citep{zhang2019theoretically} as the robust estimator. 
\begin{table}[t]
  \centering
  \begin{tabular}{c|ccc}
     & \textbf{\shortstack{Standard \\training}} & \textbf{\shortstack{Adversarial \\ training}} & \textbf{\shortstack{RST~\citep{carmon2019unlabeled}}}\\ \hline
    Robust test & $3.5\%$ & $45.8\%$ & \cellcolor{black!15} $62.5\%$ \\
 Standard test &  $95.2\%$ & $87.3\%$ &\cellcolor{black!15} $89.7\%$ 
  \end{tabular}
    \caption{Robust and standard accuracies for different training methods. Robust self-training (RST) leverages unlabeled data in addition to the \cifar~ training set to see an increase in both standard and robust accuracies over traditional adversarial training. To mitigate the tradeoff between robustness and accuracy, all we need is (possibly large amounts of) unlabeled data.}
  \label{tab:rst-cifar}
\end{table}

\subsection{Robust self-training doesn't sacrifice robustness}
\label{sec:app-rst-robust}

In Section~\ref{sec:rst}, we show that if we have access to additional unlabeled samples from the data distribution, robust self-training (RST) can mitigate the tradeoff in standard error between robust and standard estimators.
It is important that we do not sacrifice robustness in order to have better standard error.
Figure~\ref{fig:compare-rst-2} shows that in the case where robustness hurts generalization in our convex construction ($m=1$), RST improves over robust training not only in standard test error (Section~\ref{sec:rst}), but also in robust test error.
Therefore, by leveraging some unlabeled data, we can recover the standard generalization performance of standard training using RST while simultaneously improving robustness.

\begin{figure*}[t]
    \centering
    \subfigure[Test MSE]{
      \includegraphics[width=.3\textwidth]{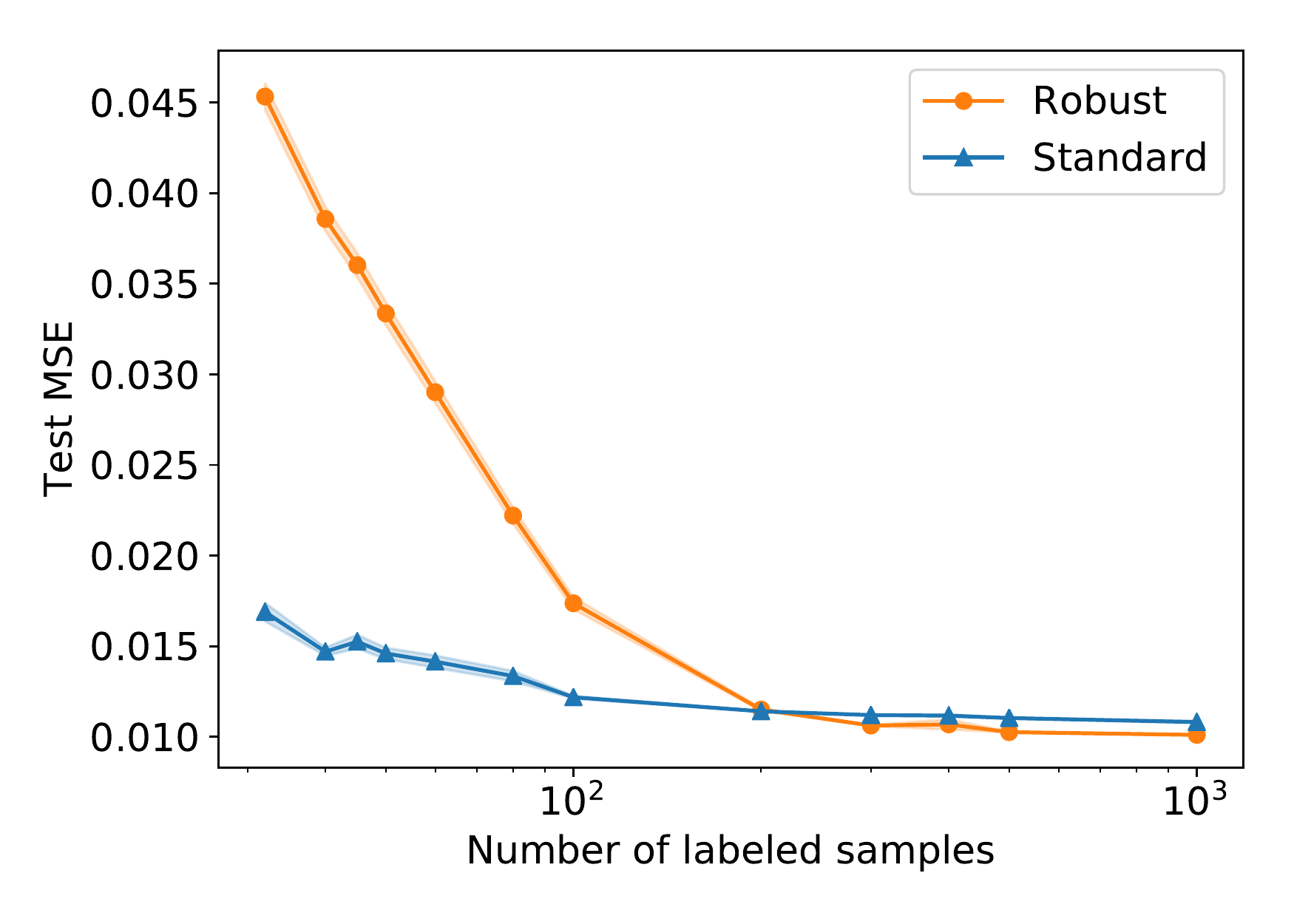}
      \label{fig:testmse_tradeoff}
  }
    \subfigure[Generalization gap]{
      \includegraphics[width=.3\textwidth]{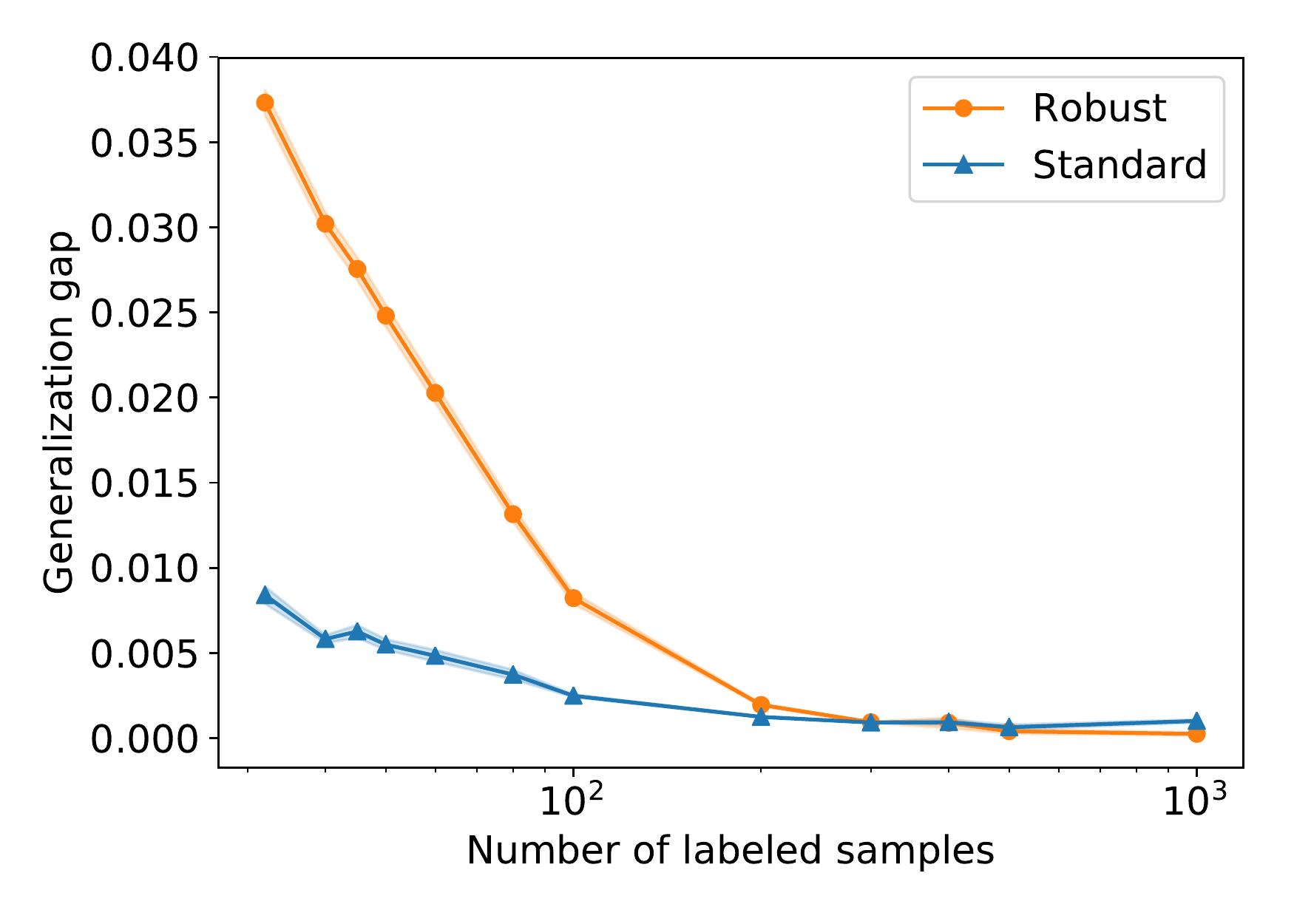}
      \label{fig:gg_tradeoff}
  }
    \subfigure[Squared norm]{
      \includegraphics[width=.3\textwidth]{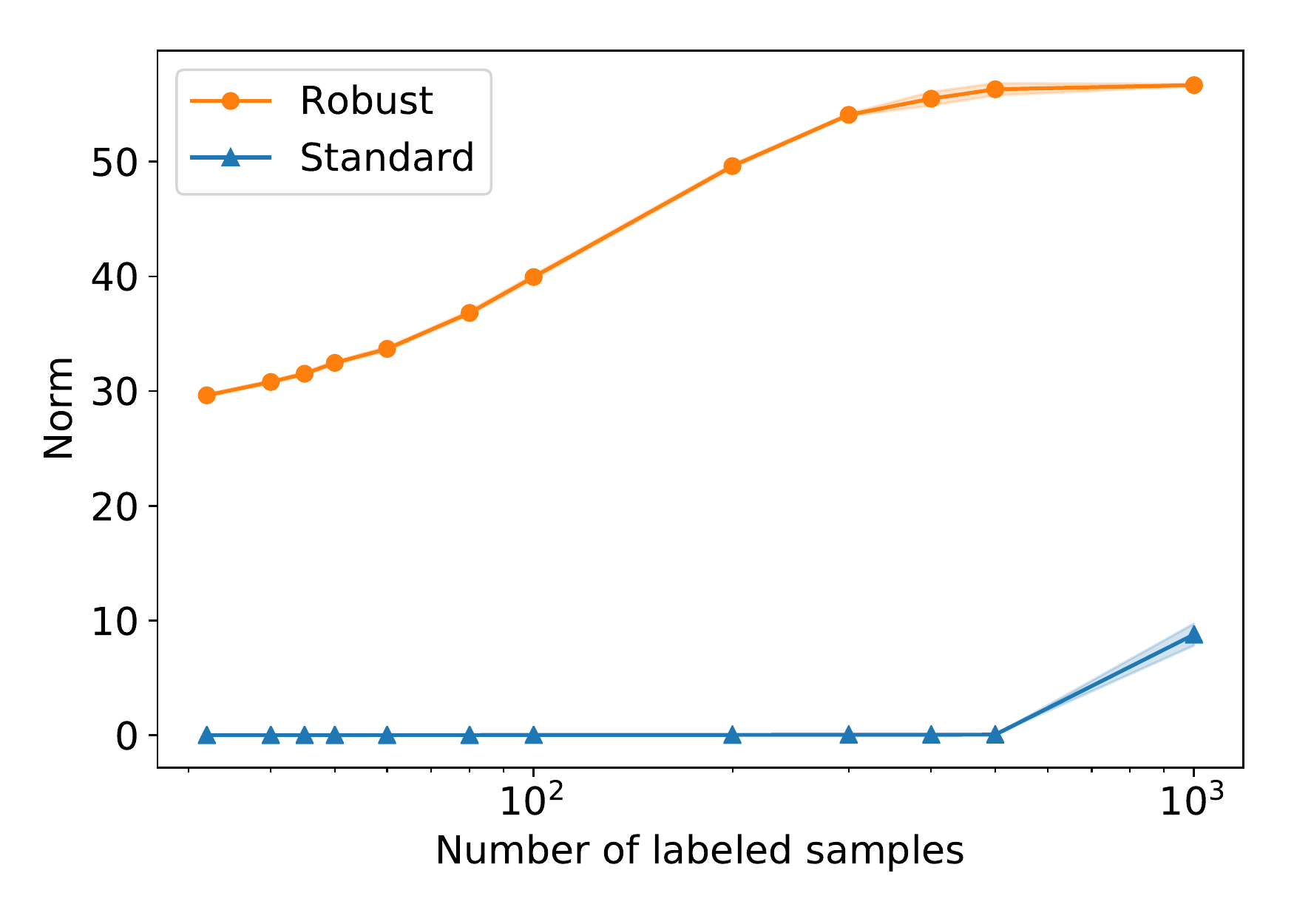}
      \label{fig:norm_tradeoff}
  }
    \subfigure[Robust train MSE]{
      \includegraphics[width=.3\textwidth]{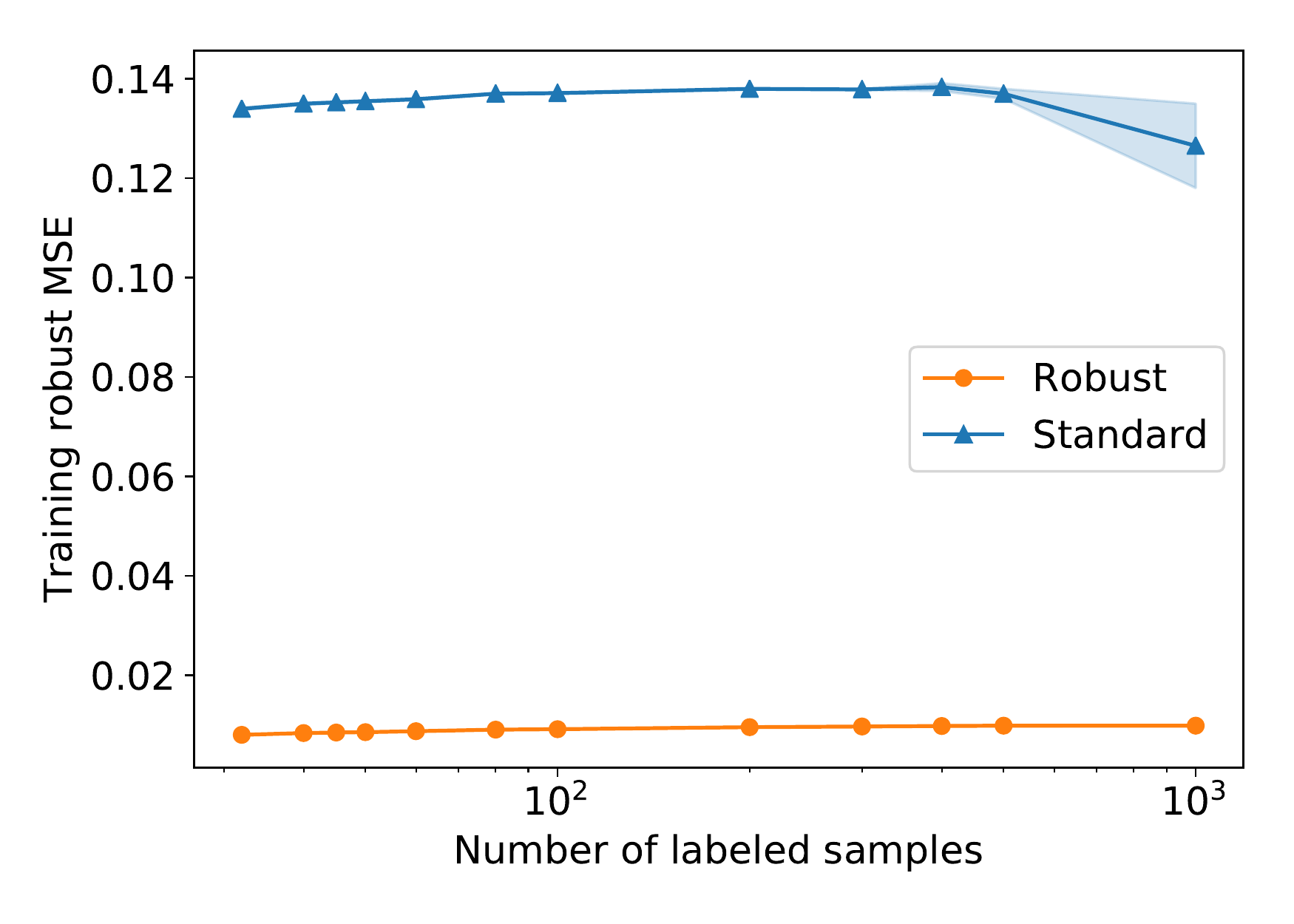}
      \label{fig:robtrain_tradeoff}
  }
    \subfigure[Robust test MSE]{
      \includegraphics[width=.3\textwidth]{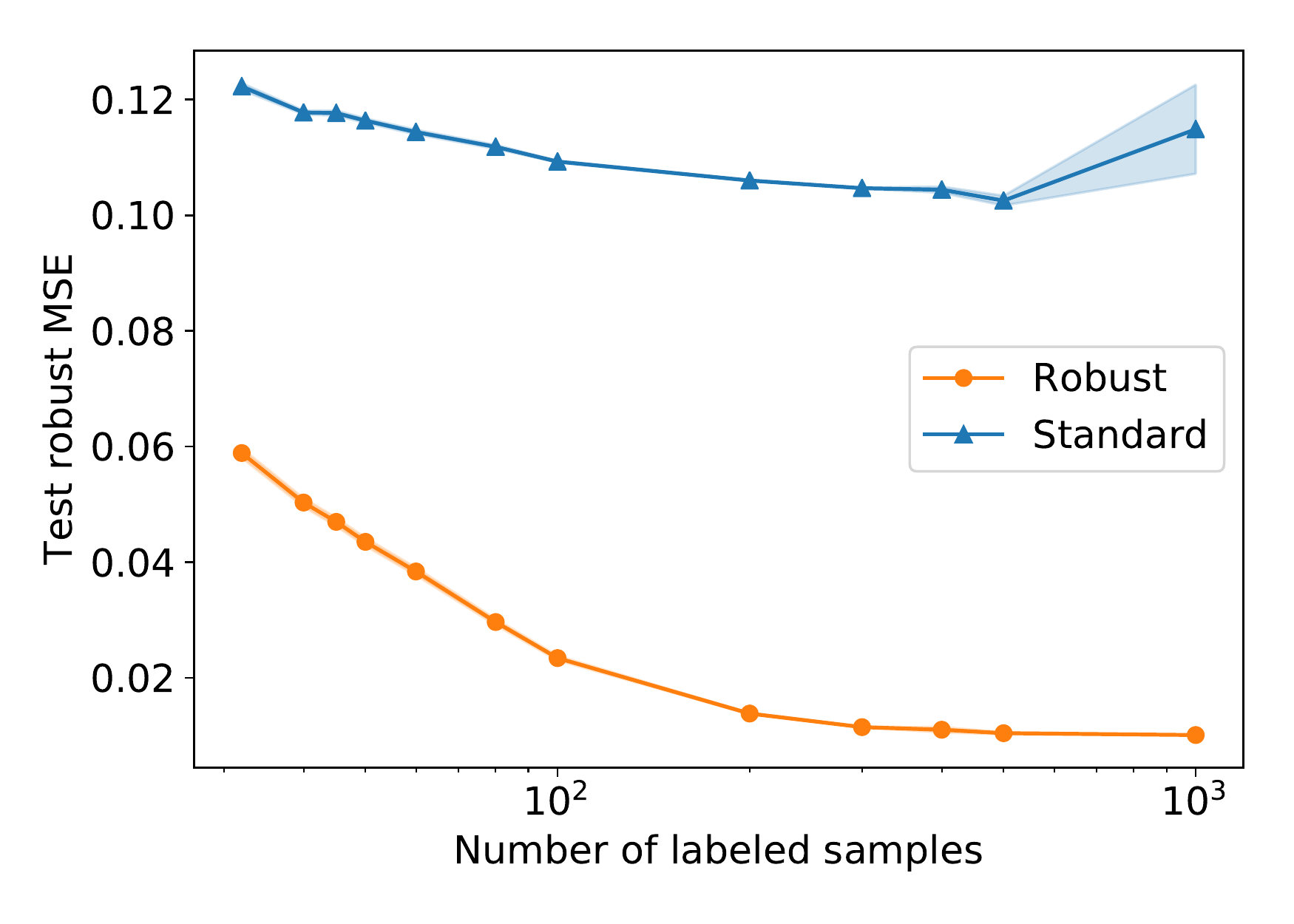}
      \label{fig:robtest_tradeoff}
  }
    \caption{Plots as number of samples varies for the case where robustness hurts ($m=1$).
    For each $n$, we pick the best regularization parameter $\lambda$ with respect to standard test MSE individually for robust and standard training.
    \textbf{(a),(b)} The standard estimator has lower test MSE, but the gap shrinks with more samples. Note that the trend in test MSE is almost identical to generalization gap. \textbf{(c)} The robust estimator has higher norm throughout training due to learning a more complex estimator. The norm of the standard estimator increases as sample size increases as it starts to converge to the true function, which is complex. \textbf{(d, e)} The robust train and test MSE is smaller for the robust estimator throughout. With larger sample size, the standard estimator improves in robust (train and test) MSE as it converges to the true function, which is robust. Shaded regions are 1 STD.
    }
    \label{fig:tradeoff-othervalues}
\end{figure*}
\begin{figure*}[t]
    \centering
    \subfigure[Test MSE]{
      \includegraphics[width=.3\textwidth]{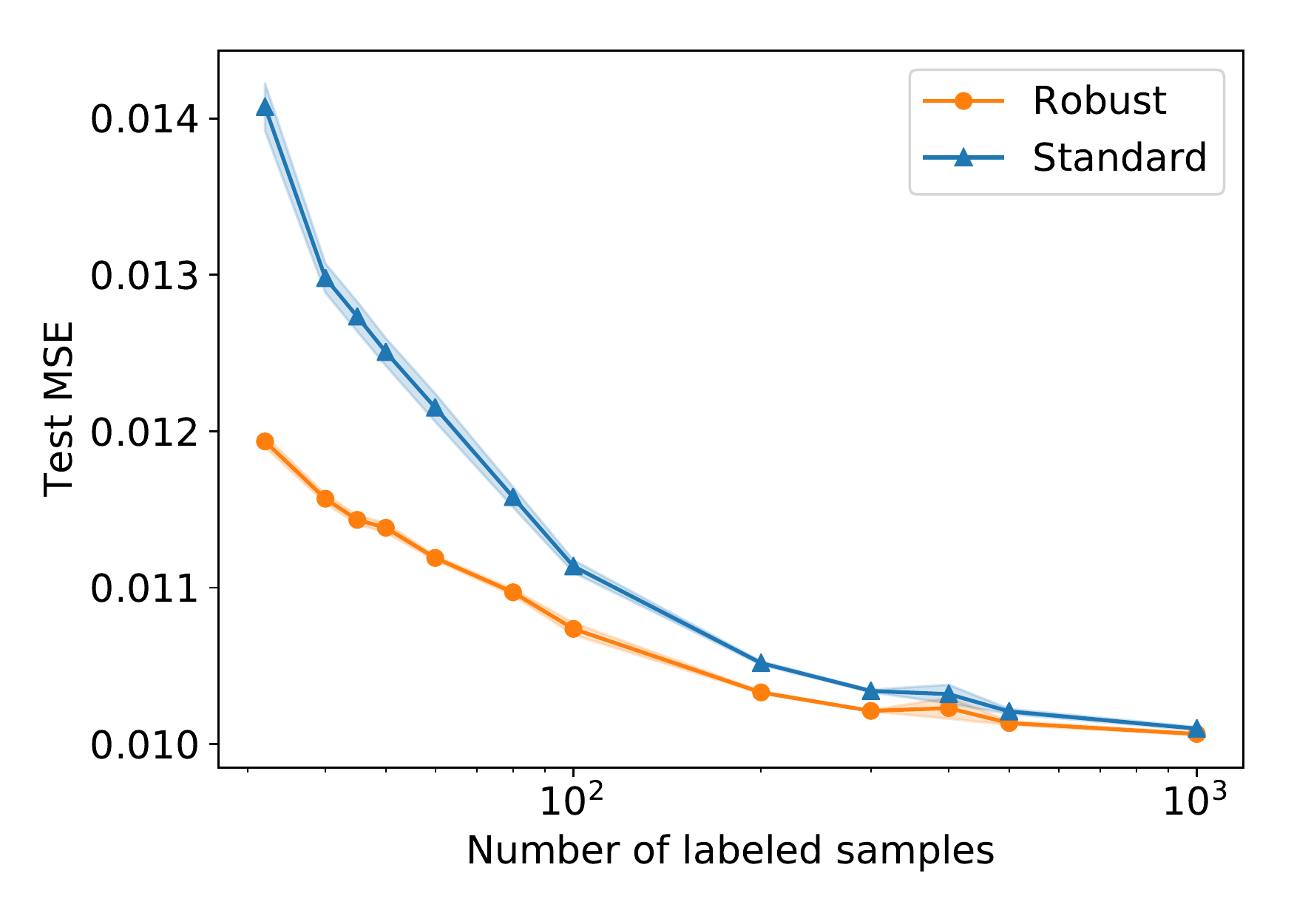}
  }
    \subfigure[Generalization gap (train MSE - test MSE)]{
      \includegraphics[width=.3\textwidth]{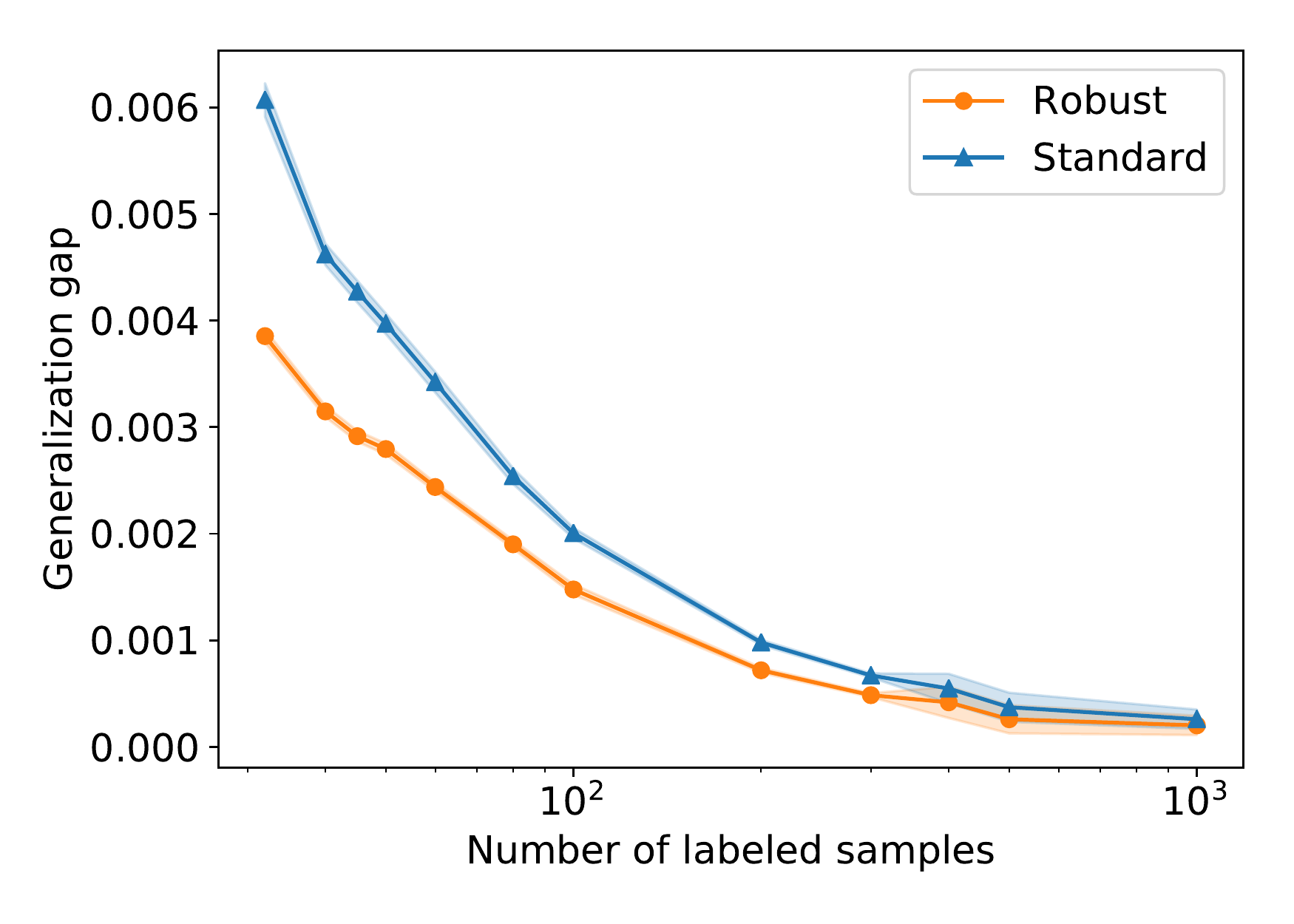}
  }
    \subfigure[Squared norm]{
      \includegraphics[width=.3\textwidth]{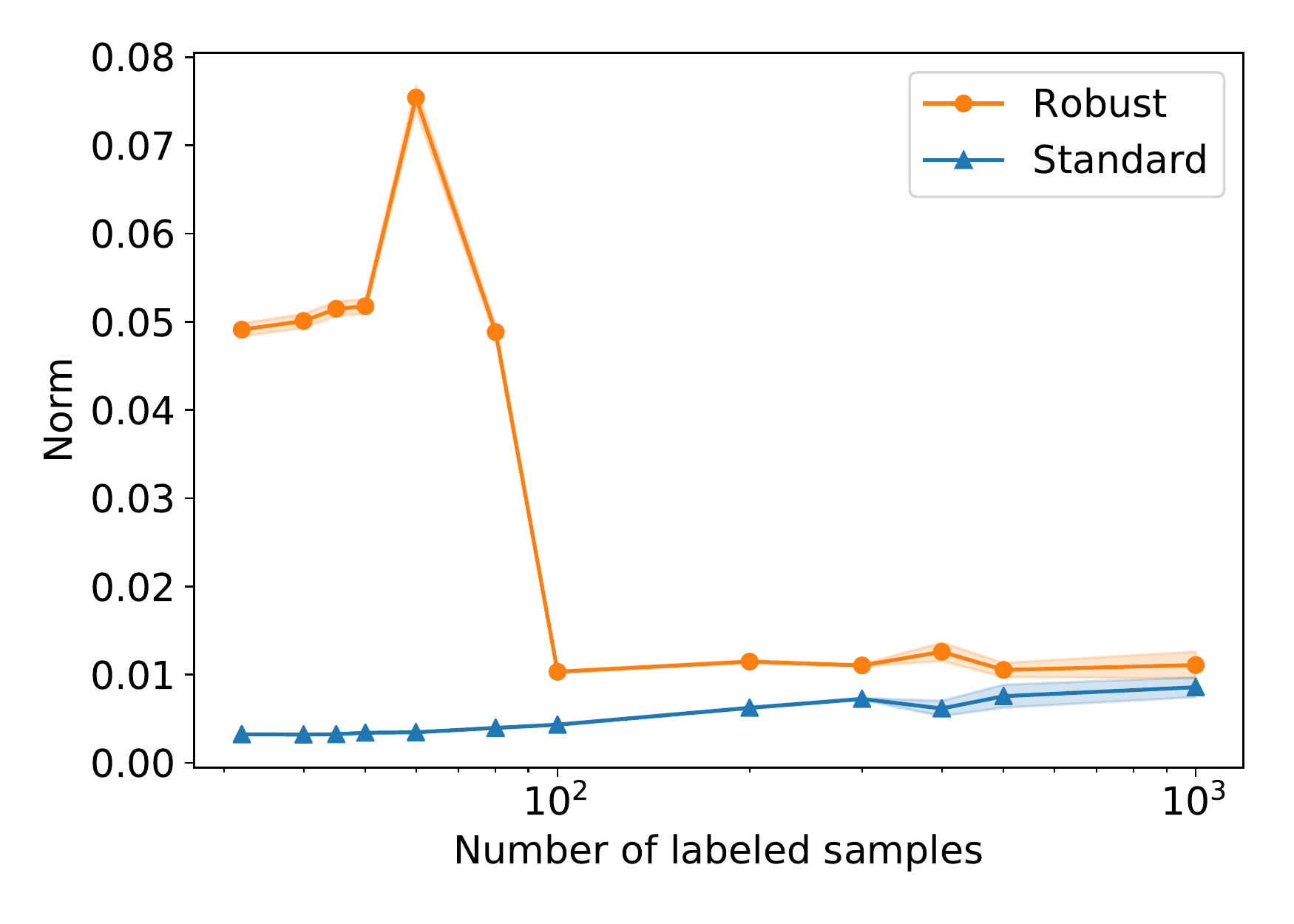}
  }
    \subfigure[Training robust MSE]{
      \includegraphics[width=.3\textwidth]{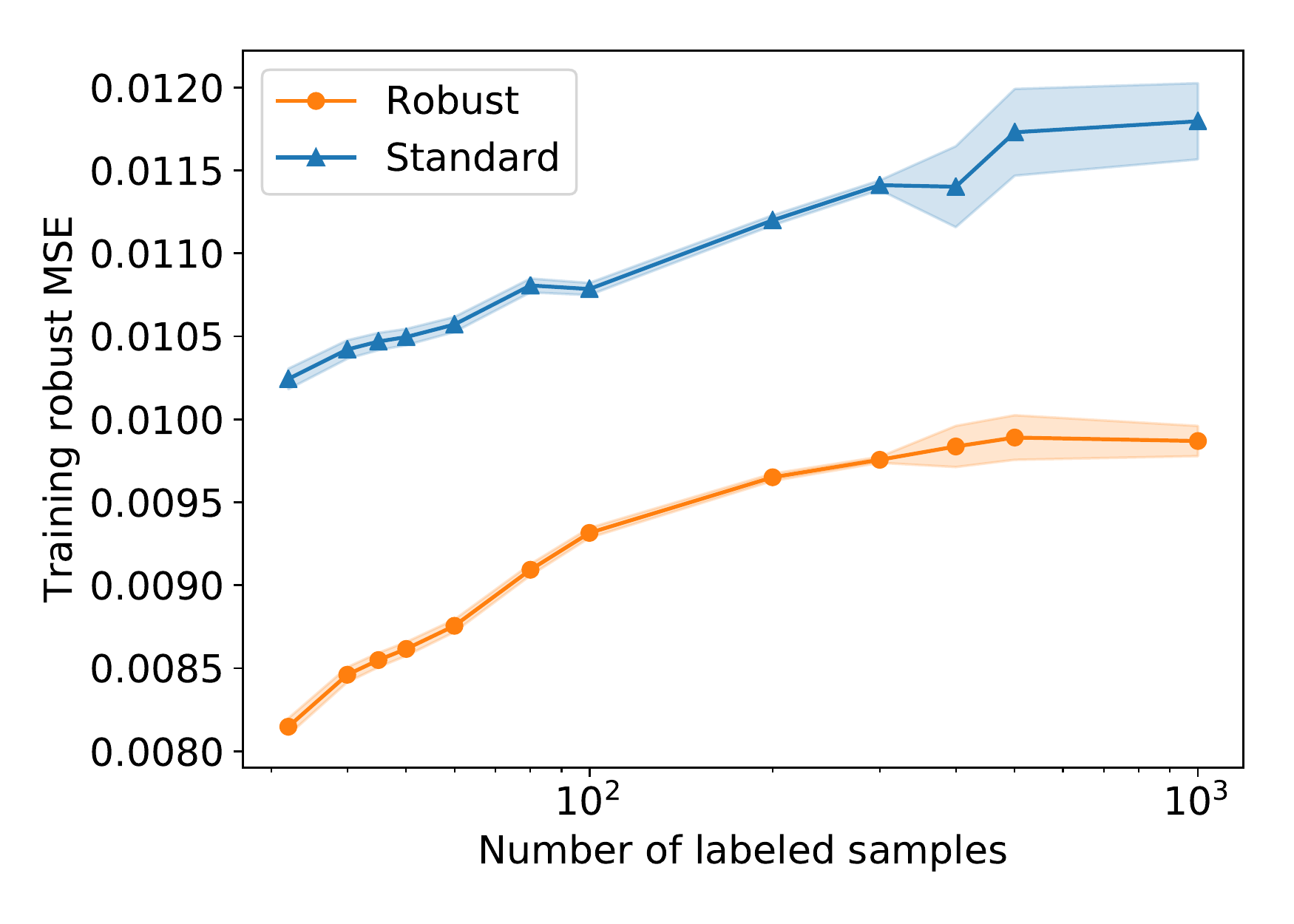}
  }
    \subfigure[Test robust MSE]{
      \includegraphics[width=.3\textwidth]{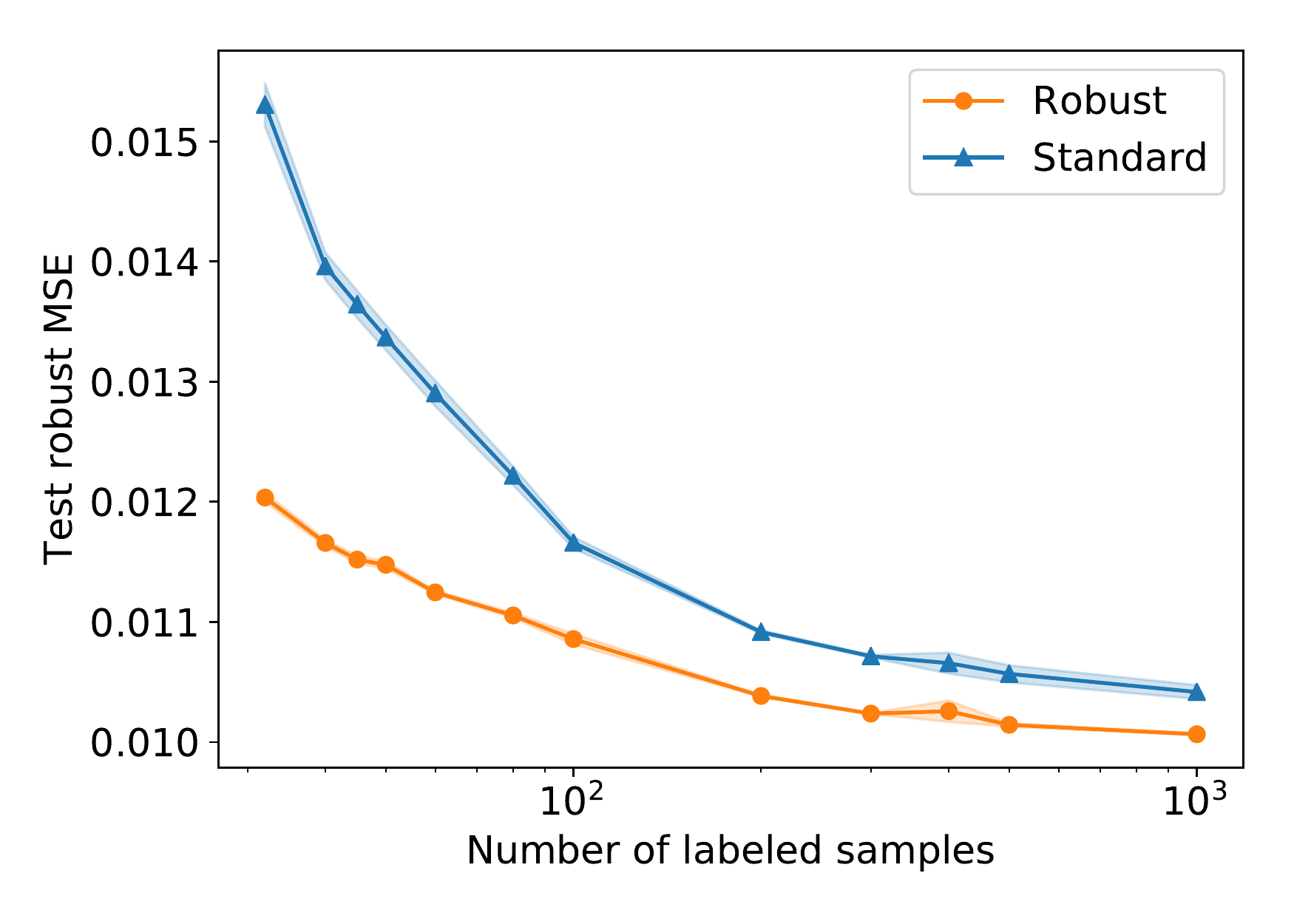}
  }
    \caption{Plots as number of samples varies for the case where robustness helps ($m=0$).
    For each $n$, we pick the best regularization parameter $\lambda$ with respect to standard test MSE individually for robust and standard training.
    \textbf{(a),(b)} The robust estimator has lower test MSE, and the gap shrinks with more samples. Note that the trend in test MSE is almost identical to generalization gap. \textbf{(c)} The robust estimator has consistent norm throughout due to the noise-cancelling behavior of optimizing the robust objective. While the standard estimator has low norm for small samples, it has high test MSE due to finding a low norm (close to linear) solution with the wrong slope. \textbf{(d, e)} The robust train and test MSE is smaller for the robust estimator throughout. Shaded regions are 1 STD.
    }
    \label{fig:tradeon-othervalues}
\end{figure*}

\section{Experimental details}
\label{sec:app-training-details}

\subsection{\cifar}
We train Wide ResNet 40-2 models~\citep{zagoruyko2016wide} using standard and adversarial training while varying the number of samples in \cifar. We sub-sample CIFAR-10 by factors of
$\{1, 2, 5, 8, 10, 20, 40\}$.
For sub-sample factors 1 to 20, we report results averaged from 2 trials each for standard and adversarial training.
For sub-sample factors greater than 20, we average over 5 trials.
We train adversarial models under the $\ell_{\infty}$ attack model with
$\ell_{\infty}$-norm constraints of sizes $\epsilon=\{1/255, 2/255, 3/255, 4/255\}$ using
PGD adversarial training~\citep{madry2018towards}. 
The models are trained for 200 epochs using minibatched gradient descent with momentum, such that $100\%$ standard training accuracy is achieved for both standard and adversarial models in all cases and
$>98\%$ adversarial training accuracy is achieved by adversarially trained models in most cases.
We did not include reuslts for subsampling factors greater than 50, since the test accuracies are very low (20-50\%).
However, we note that for very small sample sizes (subsampling factor 500), the robust estimator can have slightly better test accuracy than the standard estimator.
While this behavior is not captured by our example, we focus on capturing the observation that standard and robust test errors converge with more samples.

\subsection{\mnist}
The MNIST dataset consists of 60000 labeled examples of digits.
We sub-sample the dataset by factors of $\{1, 2, 5, 8, 10, 20, 40, 50, 80, 200, 500\}$ and report results for a small 3-layer CNN averaged over 2 trials for each sub-sample factor.
All models are trained for 200 epochs and achieve $100\%$ standard training accuracy in all cases.
The adversarial models achieve $>99\%$ adversarial training accuracy in all cases.
We train the adversarial models under the $\ell_{\infty}$ attack model with PGD adversarial training and $\epsilon=0.3$.
For computing the max in each training step, we use $40$ steps of PGD, with step size $0.01$ (the parameters used in~\citep{madry2018towards}). We use the Adam optimizer. The final robust test accuracy when training with the full training set was $~91\%$.

\paragraph{Initialization and trade-off for \mnist.}
\label{sec:app-init-mnist}
We note here that the tradeoff for adversarial training reported in~\citep{tsipras2019robustness}
is because the adversarially trained model hasn't converged (even after a large number of epochs).
Using the Xavier initialization, we get faster convergence with adversarial training and see no drop in clean accuracy at the same level of robust accuracy. Interestingly, standard training is not affected by initialization, while adversarial training is dramatically affected.

\end{document}